\documentclass{article}
\usepackage[letterpaper,top=2cm,bottom=2cm,left=3cm,right=3cm,marginparwidth=1.75cm]{geometry}

\usepackage{xcolor}
\definecolor{pantone200}{RGB}{183,18,52} 
\usepackage[colorlinks=true, linkcolor=pantone200, citecolor=pantone200, urlcolor=pantone200]{hyperref}

\usepackage{authblk}
\usepackage{orcidlink}
\usepackage[numbers]{natbib}
\usepackage{subcaption}

\usepackage{latexsym}
\usepackage{amssymb}
\usepackage{amsmath}
\usepackage{amsthm}
\usepackage{booktabs}
\usepackage{enumitem}
\usepackage{graphicx}
\usepackage{color}

\usepackage{algorithm}
\usepackage{algpseudocode}

\usepackage{comment}

\usepackage{placeins}

\usepackage{yhmath}

\usepackage[separate-uncertainty=true, group-digits=integer]{siunitx}
\usepackage{booktabs,makecell,multirow}
\renewrobustcmd{\boldmath}{}
\newrobustcmd{\B}{\fontseries{b}\selectfont}

\newtheorem{theorem}{Theorem}
\newtheorem{lemma}[theorem]{Lemma}
\newtheorem{corollary}[theorem]{Corollary}

\def\Ea{$E_{G^+}$}
\def\Eb{$E_{\mathbb{G}}$}
\def\Ec{$E_{\mathbb{G}}^*$}

\DeclareMathOperator*{\argmax}{arg\,max}

\title{Bayesian Network Structural Consensus via Greedy Min-Cut Analysis}

\author[1,2]{Pablo Torrijos \orcidlink{0000-0002-8395-3848}\thanks{Corresponding Author. Email: Pablo.Torrijos@uclm.es.}}
\author[1,2]{José M. Puerta \orcidlink{0000-0002-9164-5191}}
\author[1,3]{Juan A. Aledo \orcidlink{0000-0003-1786-8087}}
\author[1,2]{José A. Gámez \orcidlink{0000-0003-1188-1117}}

\affil[1]{Instituto de Investigación en Informática de Albacete (I3A), Universidad de Castilla-La Mancha, Albacete, Spain}
\affil[2]{Departamento de Sistemas Informáticos, Universidad de Castilla-La Mancha, Albacete, Spain}
\affil[3]{Departamento de Matemáticas, Universidad de Castilla-La Mancha, Albacete, Spain}

\date{}

\begin{document}
\maketitle

\begin{abstract}
    This paper presents the Min-Cut Bayesian Network Consensus (MCBNC) algorithm, a greedy method for structural consensus of Bayesian Networks (BNs), with applications in federated learning and model aggregation. MCBNC prunes weak edges from an initial unrestricted fusion using a structural score based on min-cut analysis, integrated into a modified Backward Equivalence Search (BES) phase of the Greedy Equivalence Search (GES) algorithm. The score quantifies edge support across input networks and is computed using max-flow. Unlike methods with fixed treewidth bounds, MCBNC introduces a pruning threshold~$\theta$ that can be selected post hoc using only structural information. Experiments on real-world BNs show that MCBNC yields sparser, more accurate consensus structures than both canonical fusion and the input networks. The method is scalable, data-agnostic, and well-suited for distributed or federated scenarios.
\end{abstract}

\begin{center}
    \begin{minipage}{0.9\textwidth}
        \small
        \centering
        \textbf{Links}
        \flushleft
        \url{https://github.com/ptorrijos99/BayesFL} (code), \\
        \url{https://doi.org/10.5281/zenodo.14917796} (datasets), \\
        \url{https://doi.org/---/---} (official proceedings version accepted in AAAI-26, without appendix), \\ 
        \url{https://arxiv.org/abs/2504.00467} (this version, including appendix). \\
    \end{minipage}
\end{center}
\vspace{-0.5\baselineskip}

%
%
\section*{Introduction}\label{sec:introduction}

\frenchspacing
Bayesian Networks (BNs)~\cite{Jensen_Nielsen,Koller_Friedman} are a formalism for modeling uncertainty probabilistically, with widespread applications in domains such as medical diagnosis~\cite{McLachlan2020}, bioinformatics~\cite{Angelopoulos_2022,Bernaola2023}, and environmental risk assessment~\cite{Dai2024}. Their semantic clarity, stemming from the encoding of conditional independencies via Directed Acyclic Graphs (DAGs), makes them particularly attractive for interpretable decision-making~\cite{MeekesRG15}. 
In many scenarios, it is necessary to aggregate multiple BNs, whether elicited from different experts or learned from disjoint datasets, into a single consensus structure. This task, known as \emph{structural fusion}~\cite{pena_finding_2011}, aims to consolidate shared independencies while minimizing model redundancy. Both BN learning and fusion are NP-hard~\cite{Jensen_Nielsen}, and naïve aggregation strategies often lead to complex models with poor inference performance.

A common approach is to compute the union of the input DAGs under a fixed node ordering~\cite{Puerta2021Fusion}, producing a dense structure that contains all independences supported by at least one input BN. Although this guarantees the definition of structural fusion, it tends to inflate the treewidth ($tw$) of the resulting network, severely limiting its practical use. The time complexity of exact inference in a BN is exponential in this $tw$, specifically $O(n \cdot k^{tw+1})$~\cite{Chandrasekaran2008}, where $n$ is the number of variables and $k$ the number of states per variable.

To address this, pruning-based methods have been studied. Genetic algorithms have been used to enforce treewidth constraints via edge deletion~\cite{Torrijos2024_CEC}, and more recently, to directly optimize consensus structure quality under user-defined objectives~\cite{Torrijos2025_GECCO}. However, these methods remain computationally expensive and require setting parameters such as target treewidth or stopping criteria, which are difficult to determine without access to data, limiting their application to scenarios such as federated learning \cite{mcmahan17aFL}.

Greedy algorithms offer a scalable alternative, but with clear limitations. \cite{Torrijos2024_CEC} also proposed a greedy pruning that approximates the unrestricted fusion; \cite{Torrijos2025_GECCO} used a similar heuristic to mimic the input graphs. Both rely on edge frequency and ignore other structural properties~\cite{Koller_Friedman}, so they serve only as initializers for genetic algorithms and fail to operate standalone. They also require a fixed treewidth bound: a value set too low removes essential edges, while a high value leaves inference intractable.

We propose a scalable, parameter-light strategy for recovering a consensus structure from input graphs without access to data. Our method begins from the unrestricted fusion obtained using the heuristic node ordering of~\cite{Puerta2021Fusion} and iteratively prunes edges based on a flow-based structural score that captures edge support across the input networks. This process prunes spurious dependencies without constraining treewidth. The only free parameter is a pruning threshold $\theta$, which can be near-optimally selected a posteriori using only the input graph structures. 

Federated learning~\cite{mcmahan17aFL} lets clients train models collaboratively without sharing private data. In the context of BNs, one natural approach~\cite{Torrijos2024_DS} is for each client to learn a local structure from its own dataset, which is then aggregated into a global consensus BN. Structural fusion is therefore the critical step, performed without data or a gold standard. Experiments in this setting confirm that our method, \emph{Min-Cut Bayesian Network Consensus} (MCBNC), consistently produces consensus structures that are not only sparser and more interpretable than those from canonical fusion but also more faithful to the underlying dependency structure than the input networks themselves on average.

\paragraph{Contributions.} The principal contributions are:
\begin{itemize}
    \item A max-flow–based score to quantify edge support.

    \item Integration of this score into the Backward Equivalence Search (BES) phase of the Greedy Equivalence Search (GES) algorithm to prune edges within the Markov equivalence class of the fused network.

    \item An adaptive pruning rule with a single threshold~$\theta$, \mbox{selected} post hoc using only input graphs.
\end{itemize}

\paragraph{Paper organization.} The paper proceeds as follows. \textbf{Background} reviews key concepts in BN fusion and flow-based analysis. \textbf{Proposal} introduces the MCBNC algorithm and its theoretical foundations. \textbf{Experimental \mbox{Methodology}} details the evaluation setup. \textbf{\mbox{Experimental} Results} report results on real and synthetic networks. \textbf{\mbox{Conclusions}} summarize findings and future directions.

%
%
\section*{Preliminaries}\label{sec:preliminaries}

\subsection*{Bayesian Networks.}
A Bayesian Network (BN) is a pair $B\!=\!(G,P)$, where $G\!=\!(V,E)$ is a directed acyclic graph (DAG) representing conditional (in)dependences over variables $V\!=\!\{v_1,\dots,v_n\}$, and $P$ is a set of probability distributions that factorizes as
\begin{equation}
    \mathbb{P}(V)=\prod_{i=1}^{n}\mathbb{P}\bigl(v_i \mid \mathbf{Pa}_G(v_i)\bigr),
\end{equation}
where $\mathbf{Pa}_G(v_i)$ denotes the parent set of $v_i$ in $G$. 
The graph $G$ encodes conditional independencies $I(G)$ via \mbox{\emph{d-separation}}~\cite{Koller_Friedman}. A DAG $G$ is an \mbox{$I$-map} of $G'$ when $I(G)\subseteq I(G')$ and is \emph{minimal} if removing any arc destroys this property. DAGs that encode the same $I(G)$ form a Markov equivalence class, representable by a Completed Partially Directed Acyclic Graph (CPDAG) $\mathcal{G}$~\cite{chickering_optimal_2002}. In $\mathcal{G}$, directed edges appear when their orientation is invariant across all equivalent DAGs; undirected edges denote ambiguity.

\paragraph{Treewidth.}
Let $\widetilde G$ be the moral graph of $G$ (all parents of each node joined and edges made undirected). The treewidth $\operatorname{tw}(G)$ is the size of the largest clique\footnote{A clique is a fully connected node subset.} in an optimal triangulation of $\widetilde G$ minus one. Exact inference is $O\bigl(n\,k^{\operatorname{tw}(G)+1}\bigr)$, where $k$ is the maximum state count per variable~\cite{Chandrasekaran2008}; low treewidth is therefore essential for BN tractability and usability.

\subsection*{Structural Fusion of Bayesian Networks} \label{subsec:fusion_bns}
Let $\{G_i = (V, E_i)\}_{i=1}^r$ be DAGs over a shared variable set $V$. A common structural fusion strategy~\cite{pena_finding_2011,Puerta2021Fusion} applies a total node ordering $\sigma$ to each $G_i$, producing acyclic DAGs $\{G_i^\sigma\}_{i=1}^r$ where all parents of a node precede it. The fused DAG is then
\begin{equation}  
    G^+ = (V, E^+), \qquad E^+ = \bigcup_{i=1}^r E_i^\sigma.
\end{equation}  
This union is guaranteed to be acyclic and is a minimal \mbox{$I$-map} of the intersection $\bigcap_i I(G_i^\sigma)$. The final density of $G^+$ depends strongly on the ordering $\sigma$, since some orderings induce fewer edges when reorienting the $G_i$. Finding the optimal $\sigma$ is NP-hard, so we adopt the heuristic from~\cite{Puerta2021Fusion}, which gives near-optimal orderings in practice.

\paragraph{From fusion to consensus.} 
Strict fusion retains all dependencies present in any input, often producing dense graphs with high treewidth, especially when the $G_i$ are heterogeneous. To address this, we define a \emph{consensus} DAG $G^* = (V, E^*)$ that maximizes a structural score:
\begin{equation}
    E^* = \argmax_{E' \in \mathcal{E}} \sum_{e \in E'} \psi(e),
\end{equation}
where $\mathcal{E}$ is a search space (e.g., subsets of $E^+$ or possible edges on $V$), and $\psi(e)$ quantifies how strongly edge $e$ is supported across the input networks. This idea was formalized in~\cite{Torrijos2025_GECCO} as an alternative to canonical fusion, enabling more interpretable and tractable structures.

\subsection*{Backward Equivalence Search (BES)} \label{subsec:bes}  
Greedy Equivalence Search (GES) is a two-phase algorithm for BN structure learning~\cite{chickering_optimal_2002}. It first adds edges in a forward phase and then removes them in a backward phase, Backward Equivalence Search (BES). Both phases operate over Markov-equivalent classes and use a decomposable score, such as Bayesian Dirichlet equivalent uniform (BDeu), to guide edge modifications. BES iteratively deletes the edge that gives the most significant score improvement. 
Formally, given a DAG $G=(V,E)$, data $D$ and the score $f(G : D)$, BES replaces $G$ by
\begin{equation}  
    G'=\argmax_{e\in E} f\bigl(G\setminus\{e\}\,:D\bigr),
\end{equation}  
and stops when no deletion increases the score. Its \textsc{Delete} operator \cite{chickering_optimal_2002} will be reused by our method.

\paragraph{Min-cut and max-flow.}
Let $D=(V,E)$ be a directed graph with non-negative capacities $c:E\to\mathbb{R}^+$. For a source $s$ and sink $t$, a cut $(S,T)$ satisfies $s\in S$, $t\in T$, $S\cup T=V$, $S\cap T=\emptyset$, and has capacity
\begin{equation}  
    \mathrm{cap}(S,T)=\sum_{u\in S,\,v\in T} c(u\to v).
\end{equation}  
The \emph{min-cut} problem seeks the cut of minimum capacity. The \emph{max-flow} problem finds a flow $f:E\to\mathbb{R}^+$ that respects capacities and flow conservation and maximises
\begin{equation}  
    \mathrm{val}(f)=\sum_{e\in\delta^+(s)} f(e).
\end{equation}  
The Max-Flow Min-Cut Theorem~\cite{Ahuja1993} states
\begin{equation}
    \max_f \mathrm{val}(f)=\min_{(S,T)}\mathrm{cap}(S,T).
\end{equation}  

\paragraph{Ford-Fulkerson algorithm.}
Any polynomial-time max-flow routine can be used. We employ the classical Ford-Fulkerson augmenting-path algorithm~\cite{Ford1956} for its simplicity. Implementation details are standard; refer to the Technical Appendix (Sec. \ref{sec:ford-fulkerson-appx}) for details.

%
%
\section*{Method: Min-Cut Bayesian Network Consensus (MCBNC)} \label{sec:proposal}
Structural fusion methods (e.g., \cite{Puerta2021Fusion}) compute a fused DAG $G^+$ that retains all (in)dependencies in the input BNs $\{B_i\}_{i=1}^r$ with structures $\{G_i\!=\!(V, E_i)\}_{i=1}^r$. While correct by construction, $G^+$ is often dense and yields high treewidth, which limits its usability. 
Our method, \emph{\mbox{Min-Cut} Bayesian Network Consensus} (MCBNC), addresses this by iteratively pruning weakly supported edges from $G^+$. The approach builds on the Backward Equivalence Search (BES) phase of Greedy Equivalence Search (GES) \cite{chickering_optimal_2002}, replacing its likelihood-based scoring with a structural score based on the max-flow min-cut algorithm. This score quantifies the support of each edge across the input graphs and enables parameterized pruning using a threshold $\theta$.
The intuition is that an edge $u\rightarrow v$ is critical only if its removal would disconnect $u$ and $v$ in the moralized ancestral subgraphs of many input DAGs. If many alternative paths exist, the min-cut is large, indicating the edge is redundant. Pruning such weakly supported edges simplifies fusion while preserving consensus dependencies.

Before pruning, $G^+$ is converted to its CPDAG $\mathcal{G}^+$ to ensure compatibility with BES operators such as \textsc{Delete} \cite{chickering_optimal_2002}.
The complete procedure is summarized in Alg.~\ref{alg:MCBNC}, with each component detailed in the subsections below. A simple example of the algorithm's execution is provided in the Technical Appendix (Sec.~\ref{sec:example}).

\begin{algorithm}[htb]
\caption{Min-Cut Bayesian Network Consensus} \label{alg:MCBNC}
\begin{algorithmic}[1]
    \Require Input DAGs $\{G_i\!=\!(V,E_i)\}_{i=1}^r$, threshold $\theta$, maximum subset size $k_{\max}$
    \Ensure  Consensus DAG $G^*$

    \State $\sigma \gets \textsc{Ordering}(\{G_i\})$ \Comment{\cite{Puerta2021Fusion}}
    \For{$i=1$ to $r$}      
        \State $G_i^\sigma \gets \textsc{MinimalIMap}(G_i,\sigma)$ \Comment{\cite{pena_finding_2011}}
    \EndFor
    \State $G^+ \gets (V,\bigcup_i E_i^\sigma)$ \Comment{Unrestricted fusion}
    \State $\mathcal{G} \gets \textsc{DAGtoCPDAG}(G^+)$ \Comment{\cite{chickering_optimal_2002}}
    
    \While{\textbf{true}}
        \State $(e^*,H^*,\Psi^*, \mathcal{C^*}) \!\gets\! \textsc{BestEdge}(\mathcal{G},\{G_i\}_{i=1}^r,k_{\max})$
        \If{$\Psi^* > \theta$} \textbf{break} \EndIf
        \State $\mathcal{G} \gets \textsc{Delete}(\mathcal{G},e^*,H^*)$ \Comment{\cite{chickering_optimal_2002}}
        \State $\{G_i \leftarrow G_i \setminus \mathcal{C}_i\}_{i=1}^r$  \Comment{Remove cut edges}
    \EndWhile
    \State $G^* \gets \textsc{PDAGtoDAG}(\mathcal{G})$ \Comment{A DAG consistent with $\mathcal{G}$}
    \State \Return $G^*$
\end{algorithmic}
\end{algorithm}

\subsection*{Edge Criticality via Min-Cut}
MCBNC prioritizes edge removals that preserve key dependencies while reducing graph complexity. To guide this, a \emph{criticality score} $\Psi^H_{(u \to v)}$ is computed from flow separation in the moralized input DAGs. The score quantifies the structural relevance of each edge $e=(u\to v)$ in the fused CPDAG $\mathcal{G}^+$. Following~\cite{chickering_optimal_2002}, deletions must preserve the Markov equivalence class. For each edge $e\!=\!(u\!\to\!v)$ in $\mathcal{G}^+$, the set of valid conditioning nodes is:
\begin{equation}  
    \mathcal{N}_{uv} = \{w \mid w\!\to\!v \text{ in } \mathcal{G}^+ \text{ and } w{-}u \text{ is undirected in } \mathcal{G}^+\,\}.
\end{equation}  

Given a candidate subset $H \subseteq \mathcal{N}_{uv}$, the criticality score $\Psi^H_{(u \to v)}$ is computed as follows (Alg.~\ref{alg:criticality}):
\begin{enumerate}
    \item For each input DAG $\{G_i\}_{i=1}^r$, extract the ancestral subgraph\footnote{The ancestral subgraph of a set $S$ in a DAG $G$ is the subgraph induced by all nodes from which there exists a directed path to some node in $S$, including the nodes in $S$ themselves.} of $\{u, v\} \cup H$, moralize it, and remove all nodes in $H$, yielding the conditioned graphs $\{\widetilde{G}_i^{H}\}_{i=1}^r$.
    
    \item On each conditioned graph $\{\widetilde{G}_i^{H}\}_{i=1}^r$, compute the size of the minimum cut separating $u$ and $v$ using the Ford-Fulkerson algorithm~\cite{Ford1956}.
    
    \item Return the average cut size across all graphs, which defines the criticality score $\Psi^H_{(u \to v)}$.
\end{enumerate}

\begin{algorithm}[htb]
\caption{\textsc{Criticality}}\label{alg:criticality}
\begin{algorithmic}[1]
    \Function{Criticality}{$(u \!\to\! v),\,\{G_i\}_{i=1}^r,\,H$}
        \For{$i = 1$ \textbf{to} $r$}
            \State $A_i \!\gets\! \textsc{AncestralSubgraph}\bigl(G_i,\{u,v\}\cup H\bigr)$
            \State $\widetilde G_i^H \gets \textsc{Moralize}(A_i)\setminus H$
            \State $S_i^H \gets \textsc{MinCut}\bigl(\widetilde G_i^H,u,v\bigr)$
        \EndFor
        \State $\Psi_{(u\to v)}^{H} \gets \frac{1}{r}\sum_{i=1}^{r}\lvert S_i^H\rvert$
        \State $\mathcal{C}_{(u\to v)}^{H} \gets \bigcup_{i=1}^{r} S_i^H$
        \State \Return $\bigl(\Psi_{(u\to v)}^{H},\,\mathcal{C}_{(u\to v)}^{H}\bigr)$
    \EndFunction
\end{algorithmic}
\end{algorithm}

Edges with lower $\Psi^{H}_{(u \to v)}$ contribute less to the structural integrity of the fused network and are prioritized for removal. This score-based strategy replaces likelihood-based criteria and avoids fixed structural constraints.

\subsection*{Greedy Edge Search} \label{subsec:greedypruning}
MCBNC performs edge deletion through a greedy search over the space of Markov equivalence classes, following the Backward Equivalence Search (BES) strategy from GES~\cite{chickering_optimal_2002}. Pruning operates on the CPDAG $\mathcal{G}^+$, where edges can be directed or undirected. Undirected edges are evaluated in both orientations $(u \to v)$ and $(v \to u)$, ensuring that all valid deletion candidates are considered.

The function \textsc{BestEdge} (Alg.~\ref{alg:bestedge}) selects, at each iteration, the least critical edge based on its structural support. For each arc $(u \to v)$, the procedure is as follows:

\begin{enumerate}
    \item Identify the valid conditioning set $\mathcal{N}_{uv}$ of nodes that are parents of $v$ and share an undirected edge with $u$ in $\mathcal{G}^+$.
    \item Generate all subsets $H \subseteq \mathcal{N}_{uv}$ of size at most $k_{\max}$, where $k_{\max}$ is a user-defined pruning budget.
    \item For each $H$, compute the criticality score $\Psi^H_{(u \to v)}$ using the method on Alg. \ref{alg:criticality}.
    \item Select the pair $(e^*, H^*)$ minimizing the score and return the edge $e^* = (u \to v)$, its score $\Psi^{H^*}_{e^*}$, the conditioning set $H^*$, and the union of cut sets $\mathcal{C}^{H^*}_{e^*}$.
\end{enumerate}

\begin{algorithm}[htb]
\caption{\textsc{BestEdge}} \label{alg:bestedge}
\begin{algorithmic}[1]
\Function{BestEdge}{$\mathcal{G}, \{G_i\}_{i=1}^r, k_{\max}$}
    \State $\Psi^* \gets \infty$
    \ForAll{$(u \to v) \in \mathcal{G}$} \Comment{$u{-}v$ $\Rightarrow$ $u\!\to\!v$, $v\!\to\!u$}
        \State $\mathcal{N}_{uv} \gets \{w \mid w \to v \text{ and } w{-}u \text{ undirected in } \mathcal{G} \}$
        \ForAll{$H \subseteq \mathcal{N}_{uv}$,\; $|H| \le k_{\max}$}
            \State $S \gets (\mathcal{N}_{uv} \setminus H) \cup (\textsc{Parents}(v,\mathcal{G}) \setminus \{u\})$
            \State $(\Psi, \mathcal{C}) \gets \textsc{Criticality}((u\!\to\!v), \{G_i\}, S)$
            \If{$\Psi < \Psi^*$}
                \State $(e^*, H^*, \Psi^*, \mathcal{C}^*) \gets ((u\!\to\!v), H, \Psi, \mathcal{C})$
            \EndIf
        \EndFor
    \EndFor
    \State \Return $(e^*, H^*, \Psi^*, \mathcal{C}^*)$
\EndFunction
\end{algorithmic}
\end{algorithm}

\subsection*{Main Iterative Pruning Scheme} \label{subsec:pruningscheme}
MCBNC removes edges from $\mathcal{G}^+$ greedily, following the BES strategy from GES~\cite{chickering_optimal_2002}. At each step, it deletes the edge with the lowest criticality score $\Psi^H_{(u \to v)}$, provided $\Psi^H_{(u \to v)} \le \theta$. The process stops when no such edge remains. Alternatively, $\theta$, the algorithm can run until $\mathcal{G}^+$ is empty, retaining the structure with minimal average structural distance to the inputs. This enables parameter-free model selection, avoiding the need for predefined treewidth bounds. The complete procedure is summarized in Alg.~\ref{alg:MCBNC}:

\begin{enumerate}
    \item Fuse the input DAGs into $G^+$ using a heuristic ordering as in~\cite{Puerta2021Fusion}.
    \item Convert $G^+$ into its CPDAG $\mathcal{G}^+$ to operate within the equivalence class using \cite{chickering_optimal_2002}.
    \item Repeatedly:
        \begin{enumerate}
            \item Use \textsc{BestEdge} (Alg.~\ref{alg:bestedge}) to find the edge $e^*$ and conditioning set $H^*$ minimizing $\Psi^H_{e^*}$.
            \item If $\Psi^H_{e^*} > \theta$, stop.
            \item Remove $e^*$ using \textsc{Delete}~\cite{chickering_optimal_2002}, update the graphs, and convert to CPDAG.
        \end{enumerate}
\end{enumerate}

\paragraph{Implementation assumptions.}
All edge capacities are assumed to be one. For each candidate edge, all conditioning subsets $H \subseteq \mathcal{N}_{uv}$ of size at most $k_{\max}$ are enumerated. This is feasible since $|\mathcal{N}_{uv}|$ is typically small, and $k_{\max}$ is fixed. The choice of max-flow algorithm is flexible; any correct implementation (e.g., Edmonds-Karp, Dinic) can be used, as the score depends only on the size of the minimum cut. Acyclicity is preserved by applying the \textsc{Delete} operator within the Markov equivalence class.

\subsection*{Properties}
This section states key properties of MCBNC.

\begin{lemma}[Monotonicity of the criticality score]
    Let $\Psi_e^{(t)}$ be the criticality score of edge $e$ after the $t$-th deletion. Then $\Psi_e^{(t+1)} \ge \Psi_e^{(t)}$ for every remaining edge $e$.
\end{lemma}
\begin{proof}
    Deleting an edge can only remove paths in the ancestral moral graphs used for computing criticality. Since the min-cut size is determined by the number of edge-disjoint paths between $u$ and $v$, its value cannot increase. Hence, the score is monotonic and non-increasing.
\end{proof}

\begin{corollary}[Score interpretation]
    Let $e = (u\!\to\!v)$ appear in exactly $k$ of the $r$ input DAGs and suppose all $u$-$v$ paths in those DAGs include $e$. Then $\Psi_e = k/r$ and:
    \[
    \theta < k/r \Rightarrow e \text{ is retained}, \quad \theta \ge k/r \Rightarrow e \text{ is removed}.
    \]
\end{corollary}

\begin{lemma}[Complexity of MCBNC with Ford-Fulkerson]\label{lem:complexity}
    Let $r$ be the number of input DAGs, $m = |E_\sigma^+|$ the number of edges in the unrestricted fusion, and $k_{\max}$ the conditioning-set cap. With unit capacities and Ford-Fulkerson for min-cut, MCBNC runs in $O\bigl(r\,m^3\,2^{k_{\max}}\bigr)$ time and $O(r\,m)$ space.
\end{lemma}
\begin{proof}
    Each min-cut takes $O(m^2)$ time. A criticality score requires $r$ min-cuts, costing $O(r\,m^2)$. For $2^{k_{\max}}$ subsets per edge and $m$ edges per iteration, \textsc{BestEdge} costs $O(r\,m^2\,2^{k_{\max}})$. The greedy loop runs at most $m$ iterations, giving total time $O(r\,m^3\,2^{k_{\max}})$. Memory is dominated by the CPDAG and $r$ DAGs, each with $O(m)$ edges.
\end{proof}

%
%
\section*{Experimental Methodology} \label{sec:methodology}
We evaluate MCBNC in both synthetic and realistic fusion settings. In both cases, the goal is to recover a consensus DAG $G^*$ that approximates a known gold-standard Bayesian network $G_\text{gs}$. Let $\{G_i\}_{i=1}^r$ denote the input DAGs, obtained either by structural perturbation of $G_\text{gs}$ (synthetic setup) or by learning from data sampled from $G_\text{gs}$ (federated setup).

As a sanity check and to replicate prior work, we first follow and extend the synthetic setup of~\cite{Puerta2021Fusion}, where each $G_i$ is derived by randomly perturbing $G_\text{gs}$. In this idealized case, MCBNC consistently reconstructs $G_\text{gs}$ with near-zero Structural Moral Hamming Distance (SMHD), even for large networks. These results confirm correctness and are reported in the Technical Appendix (Sec.~\ref{sec:appendix_synthetic}).

We then evaluate MCBNC in a more realistic and challenging federated setting. Each of the $r\!\in\!\{5, 10, 20, 30, 50, 100\}$ clients receive a private dataset $D_i$ of 5000 independent and identically distributed (i.i.d.) samples from $G_\text{gs}$ and learns a local DAG $G_i$ using the GES algorithm. The fusion operates solely on the structures $\{G_i\}_{i=1}^r$ without accessing the underlying data $\{D_i\}_{i=1}^r$. The goal is for the consensus network $G^*$ to recover the dependency structure of $G_\text{gs}$, despite the variability introduced by limited-data learning.

As gold standards, we utilize 15 benchmark networks from the \textsc{bnlearn} repository~\cite{Scutari2010BNLearn}, which cover a broad range of sizes and topologies (see Table~\ref{tab:used_BNs}).

\begin{table}[ht]
    \centering
    \setlength{\tabcolsep}{3pt}
    \caption{Benchmark Bayesian networks (nodes/edges).} \label{tab:used_BNs}
    \resizebox{0.7\textwidth}{!}{%
    \begin{tabular}{lrr@{\hspace{1.5em}}lrr@{\hspace{1.5em}}lrr}
        \toprule
        \multicolumn{1}{c}{\textsc{\bfseries Network}} & \textsc{$|V|$} & \textsc{$|E|$} &
        \multicolumn{1}{c}{\textsc{\bfseries Network}} & \textsc{$|V|$} & \textsc{$|E|$} &
        \multicolumn{1}{c}{\textsc{\bfseries Network}} & \textsc{$|V|$} & \textsc{$|E|$} \\
        \midrule
        \textsc{Asia}      &   8 &   8 & \textsc{Mildew}    &  35 &  46 & \textsc{Win95pts}  &  76 & 112 \\
        \textsc{Sachs}     &  11 &  17 & \textsc{Alarm}      &  37 &  46 & \textsc{Pathfinder} & 109 & 195 \\
        \textsc{Child}     &  20 &  25 & \textsc{Barley}     &  48 &  84 & \textsc{Andes}     & 223 & 338 \\
        \textsc{Insurance} &  27 &  52 & \textsc{Hailfinder} &  56 &  66 & \textsc{Diabetes}  & 413 & 602 \\
        \textsc{Water}     &  32 &  66 & \textsc{Hepar2}    &  70 & 123 & \textsc{Pigs}       & 441 & 592 \\
        \bottomrule
    \end{tabular}}
\end{table}

\paragraph{Experimental Protocol.}
For each benchmark network\footnote{BNs \textsc{Sachs} and \textsc{Pigs} are omitted from the main plots because GES already yields their gold-standard DAGs. Consequently, the fusion $G^+$ is optimal, and MCBNC deletes no edges for $\theta < 1$. Detailed results appear in Technical Appendix (Sec.~\ref{subsec:appx-sachs}).} and each $r \in \{5, 10, 20, 30, 50, 100\}$:

\begin{enumerate}[label=(\arabic*), leftmargin=*, nosep]
    \item A collection of $r$ datasets $\{D_i\}_{i=1}^r$ is generated by drawing $5000$ i.i.d. samples from the gold-standard BN. Each $D_i$ is used to learn a local DAG $G_i$ via GES.
    \item The input structures $\{G_i\}_{i=1}^r$ are fused into a DAG $G^+$ using the fusion method of \cite{Puerta2021Fusion}.
    \item MCBNC is executed from $G^+$, iteratively pruning edges. The algorithm produces the full trajectory $\{G^*(\theta)\}$ for all thresholds $\theta$ in a single run.
    \item Steps (2)–(3) are repeated 10 times per configuration, using the same input DAGs, to assess robustness to algorithmic randomness (e.g., tie-breaking, ordering).
    \item Each consensus DAG $G^*(\theta)$ is evaluated using multiple structural and data-based metrics.
\end{enumerate}

\paragraph{Conditioning set size.}
We fix the conditioning-set cap to $k_{\max}\!=\!10$ as an internal constant; it is not a user-tuned parameter. In practice, conditioning sets are small because they derive from nodes adjacent to both endpoints of an undirected edge in the current CPDAG, and their size shrinks as pruning progresses. Ablation results in Technical Appendix (Sec. \ref{subsec:appx-kmax}) confirm that varying $k_{\max}$ has negligible impact on consensus quality or runtime, as large sets are rarely generated.

\paragraph{Evaluation Metrics.} Each consensus DAG $G^*(\theta)$ is assessed using the following criteria:

\begin{itemize}
    \item \textbf{SMHD:} The Structural Moral Hamming Distance~\cite{Kim2019SHDMoral,Torrijos2024_CEC} quantifies structural differences after moralization. We compute the mean SMHD to the gold-standard BN (measuring fidelity) and to the input DAGs (measuring consensus). Lower values are better.

    \item \textbf{BDeu Score:} The Bayesian Dirichlet equivalent uniform score~\cite{chickering_optimal_2002} quantifies data likelihood given the structure. MCBNC ignores this criterion during pruning; we report it only for reference (larger is better).

    \item \textbf{Treewidth:} Indicates structural complexity and governs the cost of exact inference. Lower treewidths are desirable because they imply more tractable models.
\end{itemize}

Technical Appendix (Sec. \ref{sec:appendix_metrics}) provides extended metric definitions and additional structural indicators.

\subsection*{Implementation and Reproducibility} \label{subsec:reproducibility}
All code was implemented in Java (OpenJDK 17) using the \textsc{Tetrad} 7.6.5 causal inference library.\footnote{\url{https://github.com/cmu-phil/tetrad/releases/tag/v7.6.5}} Structure learning was performed with GES. All real-world networks were obtained from the \textsc{bnlearn} repository (see Table~\ref{tab:used_BNs}). Experiments were run on Intel Xeon E5-2650 (8 cores) with 32~GB RAM per run. To ensure full reproducibility, we provide all source code, experiment scripts, and preprocessed datasets on GitHub.\footnote{\url{https://github.com/ptorrijos99/BayesFL}} The datasets are also archived on Zenodo.\footnote{\url{https://doi.org/10.5281/zenodo.14917796}} Statistical tests were carried out using the \textsc{exreport} package~\cite{exreport} for R.

%
%
\section*{Experimental Results} \label{sec:results}  
We present the results of applying MCBNC in the federated learning scenario. Each figure plots performance metrics as a function of the fusion threshold $\theta$. The leftmost point corresponds to the initial fusion $G^+$~\cite{Puerta2021Fusion}, while the rightmost reflects the empty network.

\subsection*{Structural Accuracy (SMHD)}
Fig.~\ref{fig:exp3-SMHD-gold-standard} shows how SMHD of $G^*$ to the gold-standard BN $G_\text{gs}$ varies with the pruning threshold $\theta$ (from $G^+$ on $\theta\!=\!0$ to $\emptyset$ on the last $\theta$). In almost all cases, $G^+$ yields worse SMHD than even the empty DAG, confirming that unrestricted fusion accumulates spurious dependencies and the need for consensus fusions.
Applying MCBNC yields steep SMHD reductions\footnote{An exception is the \textsc{Pathfinder} BN, where SMHD improves monotonically even as the network is pruned to near emptiness. This reflects a structural mismatch in the input DAGs, as GES fails to recover the underlying semi-Naive Bayes structure. This limitation is known in the literature~\cite{Laborda2024_KBS}.}, particularly in large networks like \textsc{Andes} or \textsc{Diabetes}, where improvements over $G^+$ span up to two orders of magnitude. Gains relative to the GES-generated input DAGs are also notable, as MCBNC removes dataset-specific artifacts and consolidates shared dependencies, resulting in BNs that are more similar to $G_\text{gs}$. Performance remains stable across a broad range of $\theta$ values, with over-pruning (and SMHD degradation) occurring near $\theta=1$.

\begin{figure*}[htb]
    \centering
    \includegraphics[width=1\linewidth]{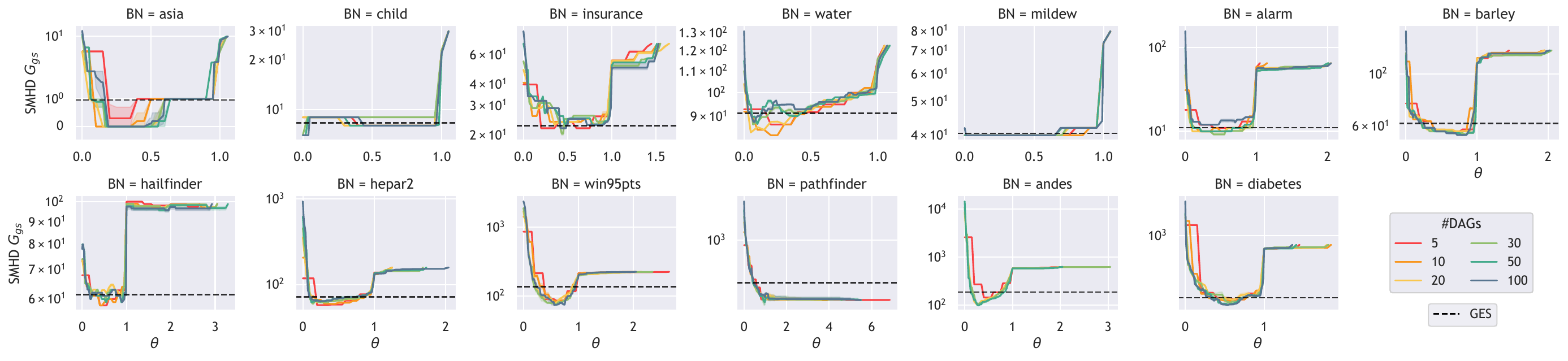}
    \caption{Mean SMHD to the gold-standard BN $G_{\text{gs}}$ across thresholds $\theta$ for each BN. Leftmost point: full fusion $G^+$. Rightmost: empty DAG $\emptyset$. Horizontal line: average SMHD of input BNs from GES to $G_{\text{gs}}$. Lower is better.}
    \label{fig:exp3-SMHD-gold-standard}
\end{figure*}

\subsection*{Data Fit (BDeu Score)}
Fig.~\ref{fig:exp3-BDeu} reports the BDeu scores of the consensus networks across different values of $\theta$. GES optimizes BDeu directly, so its input DAGs perform strongly. MCBNC, by contrast, neither accesses the data nor optimizes any likelihood-based objective. Still, it achieves scores comparable to (and occasionally exceeding) those of the input networks. In some cases, such as the \textsc{Barley} and \textsc{Mildew} BNs, even the gold-standard structure yields lower BDeu. This well-known phenomenon arises because sparser graphs, which correctly reflect the true dependencies, may underfit finite datasets. In \textsc{Pathfinder}, MCBNC again outperforms the gold standard in BDeu, but this does not imply a better structure: SMHD remains high (Fig.~\ref{fig:exp3-SMHD-gold-standard}), confirming that BDeu and structural accuracy do not always align. Overall, MCBNC achieves competitive BDeu scores despite being data-agnostic. Still, selecting an appropriate fusion threshold $\theta$ is crucial.

\begin{figure*}[htb]
    \centering
    \includegraphics[width=1\linewidth]{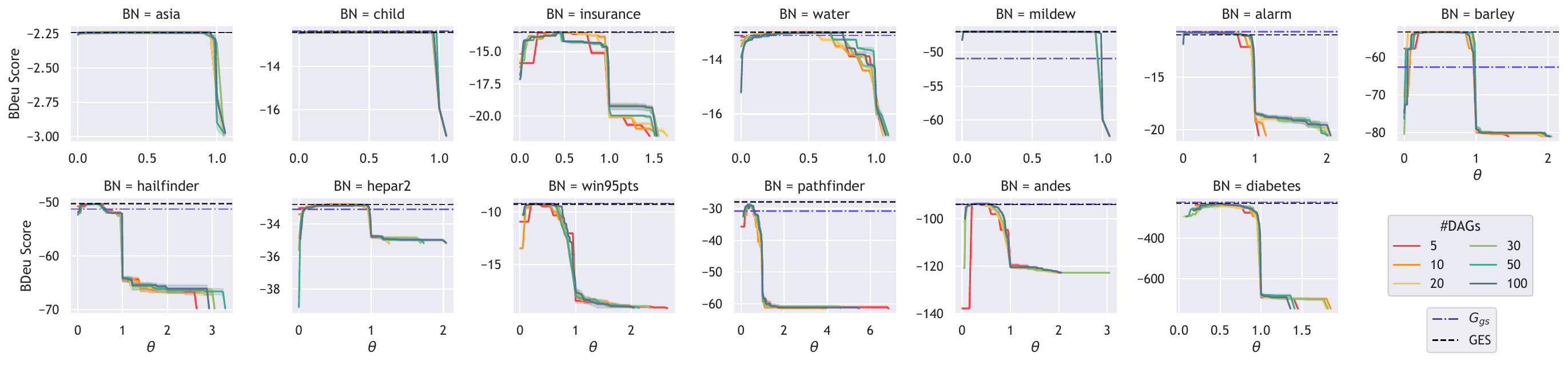}
    \caption{Mean BDeu score across thresholds $\theta$ for each BN. Leftmost point: full fusion $G^+$. Rightmost: empty DAG $\emptyset$. Horizontal lines: average of input BNs from GES (black) and gold-standard BN (purple). Higher is better.}
    \label{fig:exp3-BDeu}
\end{figure*}

\subsection*{Choosing the Fusion Threshold $\theta$}
\label{subsec:theta}
Detailed SMHD-BDeu curves for each client count $r\in\{5,10,20,30,50,100\}$ are reported in Technical Appendix (Sec.~\ref{subsec:appx-threshold}). These plots show that the threshold $\theta$ minimizing SMHD to the input DAGs also tends to maximize structural agreement with the gold-standard network and yields strong BDeu scores. This supports a practical selection strategy: set $\theta$ post hoc to minimize the mean SMHD to the input graphs. This criterion requires no access to data or ground truth, making it suitable for realistic scenarios such as federated learning. Rather than displaying the six SMHD–BDeu curves, we summarize the evidence statistically below.

\paragraph{Method.}
For each benchmark BN and each $r$, we extracted the consensus DAG $G^*(\theta)$ on the point $\theta$ that minimized SMHD to the input GES DAGs $\{G_i\}_{i=1}^r$. Three algorithms were compared: (i) MCBNC ($G^*$) at the selected $\theta$, (ii) the average of the $r$ GES DAGs, and (iii) the unrestricted fusion $G^{+}$. Ranks over benchmarks were analysed with the Friedman test~\cite{Friedman1940} to assess whether all methods perform equally. If the null hypothesis was rejected, pairwise differences were tested using Holm’s post-hoc correction~\cite{Holm1979}. Both tests used $\alpha = 0.01$, following standard practice~\cite{exreport-fuente1,exreport-fuente2}.

\begin{table}[tbp]
 \centering
  \caption{Statistical comparison over 15 BNs and six client counts ($90$ cases). Lower rank is better. $p$-values refer to Holm’s procedure against the top-ranked method; bold values indicate \textbf{non-rejection} of $H_0$ at $\alpha = 0.01$.} \label{tab:stat}
  \resizebox{0.7\columnwidth}{!} {%
    \begin{tabular*}{0.7\columnwidth}{l@{\extracolsep{\fill}}l@{\extracolsep{\fill}}S[table-format=1.2]@{\extracolsep{\fill}}S[table-format=1.2e-2]@{\extracolsep{\fill}}c}
      \toprule
      \textsc{Metric} & \textsc{Method} & \textsc{Rank} & \textsc{$p$-value} & \textsc{W\,/\,T\,/\,L} \\
      \midrule
      \multirow{3}{*}{\textsc{SMHD}}
        & \textbf{MCBNC ($G^*$)}                            & \B 1.40 & {---}              & {---} \\
        & GES $\left(\{\overline{G_i}\}_{i=1}^r\right)$     & 1.91    & 6.95e-04           & 61\,/\,13\,/\,16 \\
        & Fusion ($G^+$)                                    & 2.69    & 7.68e-18           & 68\,/\,17\,/\,5 \\
      \midrule
      \multirow{3}{*}{\textsc{BDeu}}
        & GES $\left(\{\overline{G_i}\}_{i=1}^r\right)$     & \B 1.58 & {---}              & {---} \\
        & \textbf{MCBNC ($G^*$)}                              & 1.70    & \B 4.12e-01        & 48\,/\,9\,/\,33 \\
        & Fusion ($G^+$)                                    & 2.72    & 3.25e-14           & 71\,/\,9\,/\,10 \\
      \bottomrule
    \end{tabular*}
  }
\end{table}

\paragraph{Interpretation.}
The Friedman test rejects the null hypothesis of equal methods for both metrics: $p\!=\!2.32\!\times\!10^{-17}$ for SMHD and $p\!=\!3.66\!\times\!10^{-16}$ for BDeu. Holm’s post-hoc analysis (Table~\ref{tab:stat}) confirms that, for SMHD, MCBNC significantly outperforms both the GES average and the unrestricted fusion. Among the ties, 12 correspond to \textsc{Sachs} and \textsc{Pigs}, where GES already recovers $G_{\text{gs}}$ and no structural improvement is possible. The rest occur in small networks, where differences are minor. For BDeu, MCBNC and GES are statistically indistinguishable ($p\!\approx\!0.41$), while both significantly outperform the unrestricted fusion. This is expected: GES optimizes and overfits BDeu, whereas MCBNC still yields competitive likelihood. These results confirm that selecting $\theta$ by minimizing SMHD to the input GES DAGs yields consensus networks that are structurally faithful and competitive in terms of data fit.

\subsection*{Structural Properties of the Fused Networks}
Fig.~\ref{fig:exp3-TW} plots the treewidth of the consensus BNs as $\theta$ varies (edge-count curves are in Technical Appendix C.1). Pruning with small~$\theta$ eliminates many weak edges, producing an immediate and drastic drop in treewidth. For $\theta\!\in\![0.2,0.8]$ the curve flattens: MCBNC has removed most surplus edges yet still preserves the backbone of dependencies.  
Beyond $\theta\!\approx\!0.9$, relevant edges vanish and treewidth falls again, mirroring the rise in SMHD. The vertical dotted lines mark the selected $\theta$ for each number of clients. At those points, the consensus graphs are never denser (and are frequently sparser) than both the gold-standard and the individual GES models, despite matching or surpassing them in SMHD. Networks such as \textsc{Win95pts} illustrate the benefit: the treewidth drops from approximately $20$ to around $10$, while the mean SMHD to the gold standard improves by $58.7\%$ (Fig. \ref{fig:exp3-SMHD-gold-standard}). 

\begin{figure*}[htb]
    \centering
    \includegraphics[width=1\linewidth]{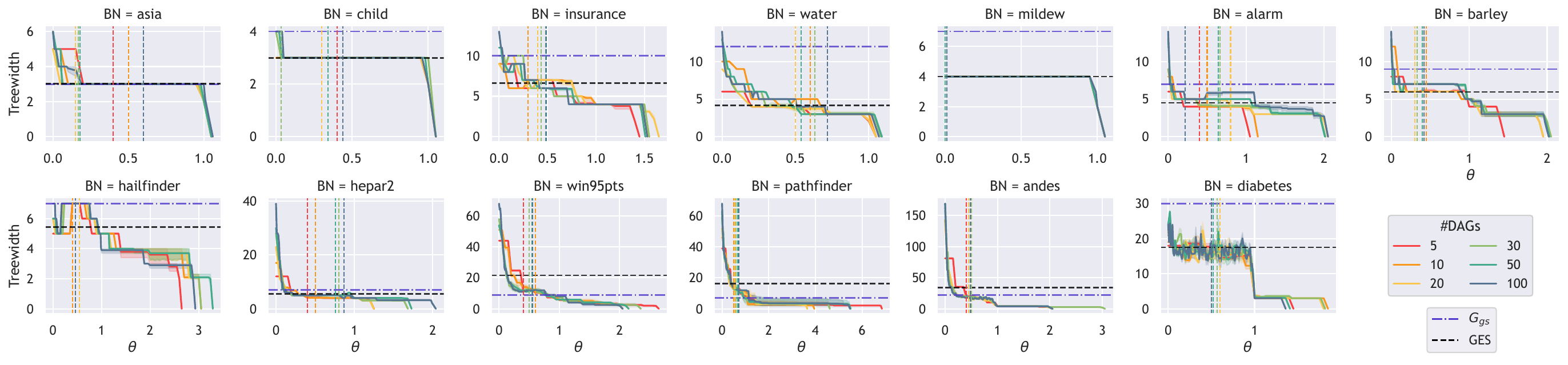}
    \caption{Mean treewidth across pruning thresholds $\theta$ for each BN. Dashed lines: selected~$\theta$ for each $\#$DAGs based on SMHD w.r.t. input BNs. Horizontal lines: average of input BNs from GES (black) and gold-standard BN (purple).}
    \label{fig:exp3-TW}
\end{figure*}

\subsection*{Runtime Comparison with Prior Methods}
Figure~\ref{fig:exp5-timeComparison} shows the runtime of MCBNC compared to the genetic fusion algorithms from~\cite{Torrijos2024_CEC,Torrijos2025_GECCO}, using the same networks and number of input DAGs as in those studies. The algorithm in~\cite{Torrijos2024_CEC} searches over the set \Ea{}, corresponding to arcs in the unrestricted fusion. The method in~\cite{Torrijos2025_GECCO} generalizes this by operating over \Eb{} (all input edges, with repetition) or \Ec{} (without repetition), depending on the chromosome encoding. Despite these differences, all genetic variants show similar scaling. MCBNC is several orders of magnitude faster, making it impractical to replicate our complete evaluation with these algorithms. The reliance on a fixed treewidth in the other methods complicates fair comparisons, as no unique treewidth target applies across networks or aggregation levels. Complete runtime results for MCBNC are provided in the Technical Appendix (Sec.~\ref{subsec:appx-additional}).

\begin{figure}[htb]
    \centering
    \includegraphics[width=0.7\linewidth]{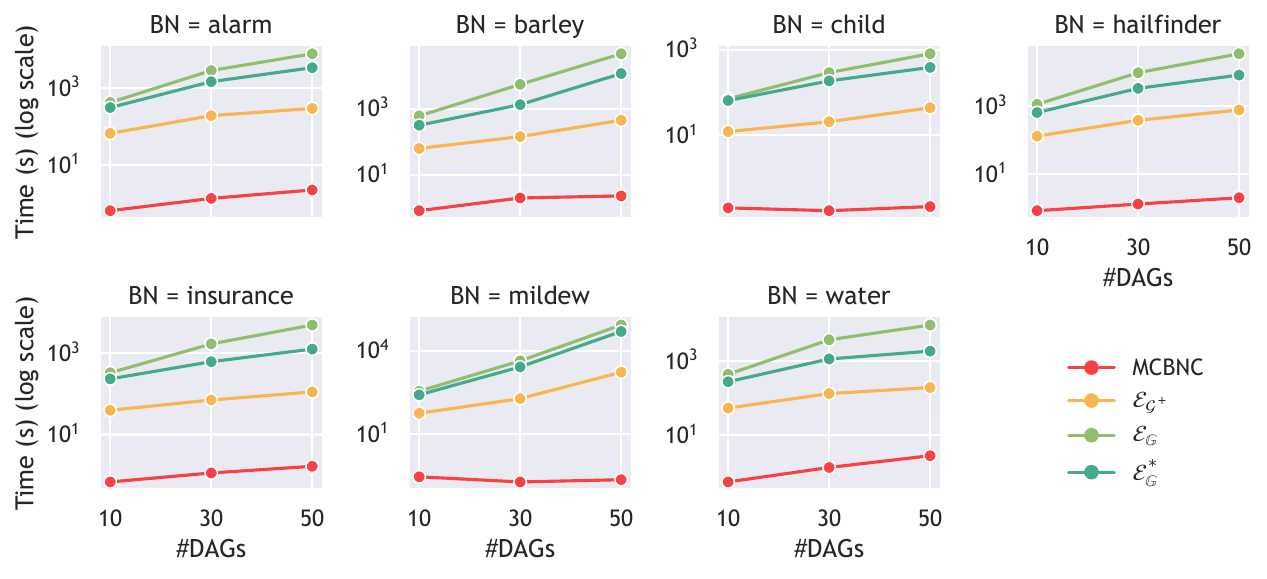}
    \caption{Total execution time vs. number of input DAGs.}
    \label{fig:exp5-timeComparison}
\end{figure}

%
%
\section*{Conclusions}\label{sec:conclusions}
This work introduced the Min-Cut Bayesian Network Consensus (MCBNC) algorithm for structure-level fusion of Bayesian networks. MCBNC overcomes limitations of existing fusion methods by pruning non-essential edges using a backward strategy guided by min-cut analysis. Unlike unrestricted fusion~\cite{Puerta2021Fusion}, which preserves all independencies at the cost of excessive complexity, or bounded approaches requiring a user-defined treewidth~\cite{Torrijos2024_CEC,Torrijos2025_GECCO}, MCBNC offers an interpretable and tunable alternative based on a single threshold $\theta$. Empirically, it consistently yields consensus networks that outperform both the unrestricted fusion and the input BNs in structural fidelity (SMHD), while being simpler and achieving competitive BDeu scores, all without accessing any data. The pruning threshold~$\theta$ can be near-optimally selected using only structural information, making MCBNC applicable in realistic settings.
These properties make MCBNC well-suited to federated scenarios, where local models are learned independently and no data sharing is allowed. It assumes identical node sets. Extending the flow-based score to mixed or evolving variable sets, studying robustness under non-i.i.d. client distributions, and integrating secure aggregation protocols are immediate directions for future research. Additionally, embedding MCBNC into advanced federated frameworks is a promising direction for future research.


\bibliographystyle{ieeetr}
\bibliography{mybibfile}

\section*{Acknoledgements}
This work was supported by SBPLY/21/180225/000062 (Junta de Comunidades de Castilla-La Mancha and ERDF A way of making Europe); PID2022-139293NB-C32 (MICIU/AEI/10.13039/501100011033 and ERDF, EU); FPU21/01074 (MICIU/AEI/10.13039/501100011033 and ESF+); 2025-GRIN-38476 (Universidad de Castilla-La Mancha and ERDF A way of making Europe); TED2021-131291B-I00 (MICIU/AEI/10.13039/501100011033 and European Union NextGenerationEU/PRTR).

Camera-ready version accepted at AAAI-26. The official proceedings version is published with DOI: \url{https://doi.org/---/---}.


\appendix
\renewcommand{\thesection}{\Alph{section}}
\renewcommand{\thefigure}{\thesection\arabic{figure}}
\setcounter{figure}{0}
\setcounter{table}{0}

\section*{Technical Appendix}
\addcontentsline{toc}{section}{Technical Appendix}

This technical appendix supplements the paper ``Bayesian Network Structural Consensus via Greedy Min-Cut \mbox{Analysis}'' and includes:
\begin{itemize}
    \item Appendix A: Extended Metric Definitions.
    \item Appendix B: Experimental Evaluation with Synthetic BNs.
    \item Appendix C: Extended Federated Learning Experimental Results.
    \item Appendix D: Ford-Fulkerson Algorithm.
    \item Appendix E: Illustrative Example of MCBNC Algorithm.
\end{itemize}

%
%
\section{Extended Metric Definitions} \label{sec:appendix_metrics}

This section provides extended definitions and interpretations of the evaluation metrics used to assess the quality of the consensus Bayesian Networks (BNs) generated by MCBNC. In addition to the structural and data-fit metrics introduced in Section 4.2, we also report edge count and execution time, offering a broader characterization of model complexity and scalability.

\paragraph{Structural Moral Hamming Distance (SMHD):}  
SMHD quantifies the structural similarity between two networks by comparing their moral graphs, which better reflect the conditional independencies encoded in the DAGs. Unlike the Structural Hamming Distance (SHD), which counts directed arc differences, SMHD measures discrepancies in undirected moralized structures \cite{Kim2019SHDMoral,Torrijos2024_CEC}. A lower SMHD indicates closer agreement in the underlying dependency structure. We compute SMHD in two ways: (i) relative to the input DAGs $\{G_i\}_{i=1}^r$, used as a proxy for training accuracy, and (ii) relative to the gold-standard BN, interpreted as a test-time score. The fusion process is deemed structurally beneficial if the consensus BN improves over the average GES networks in SMHD to the gold standard.

\paragraph{Bayesian Dirichlet equivalent uniform (BDeu) Score:}  
This score measures how well a BN structure fits observed data \cite{chickering_optimal_2002}. Although commonly used in structure learning, BDeu is known to be sensitive to overfitting, as it rewards structures that match the empirical distribution, even when the encoded independencies are not meaningful. MCBNC does not optimize for BDeu directly (it has no access to the data), so BDeu is used purely as a post-hoc evaluation metric. Higher scores indicate a better data fit, but must be interpreted in conjunction with structural metrics, such as SMHD.

\paragraph{Number of Edges:}  
Edge count offers a coarse but useful measure of structural complexity. Dense networks tend to be harder to interpret and may capture spurious relationships (overfitting), while too few edges may miss critical dependencies (underfitting). Although edge count often correlates with treewidth, this is not guaranteed; some sparse networks can still have high treewidth due to their specific connectivity patterns.

\paragraph{Treewidth:}  
Treewidth reflects the tractability of inference in a BN. It is defined as the size of the largest clique in a triangulated moral graph minus one. Exact inference is exponential in the treewidth, so minimizing it is desirable. In contrast to prior fusion methods that impose explicit treewidth bounds \cite{Torrijos2024_CEC,Torrijos2025_GECCO}, MCBNC reduces treewidth organically through pruning, without fixing a maximum bound.

\paragraph{Execution Time:}  
We also measure the cumulative execution time of MCBNC as a function of $\theta$. Runtime reflects the algorithm's scalability and depends on both the number of input DAGs and the size/complexity of each BN. Most of the computation is concentrated in early iterations (i.e., low $\theta$), where more pruning occurs. Runtime is also sensitive to the maximum subset size $k_{\max}$ used for evaluating min-cut criticality. Reducing this parameter can improve performance if necessary.

%
%
\section{Experimental Evaluation with Synthetic BNs} \label{sec:appendix_synthetic}

To validate our method and replicate the original fusion study from \cite{Puerta2021Fusion}, we conduct synthetic experiments using their generation protocol, which has been extended to larger networks and more input DAGs. Each experiment begins with a base DAG $G_0$ generated randomly with $n \in \{10, 30, 50, 100\}$ nodes. From this ground-truth network, we derive $r \in \{10, 30, 50, 100\}$ input DAGs $\{G_1, \dots, G_r\}$ by applying $p = n \cdot 0.75$ random structural perturbations per DAG. Each perturbation randomly adds or deletes an edge $x \to y$, ensuring that the resulting graph remains acyclic. We enforce structural constraints during this process to maintain a maximum of three parents and four children per node, and a total of at most $e = n \cdot 2.5$ edges per graph. These perturbed DAGs serve as input for the MCBNC algorithm, which produces a consensus structure $G^*$. 

\begin{figure}[htb]
    \centering
    \begin{subfigure}[b]{0.49\textwidth}
        \centering
        \includegraphics[width=\linewidth]{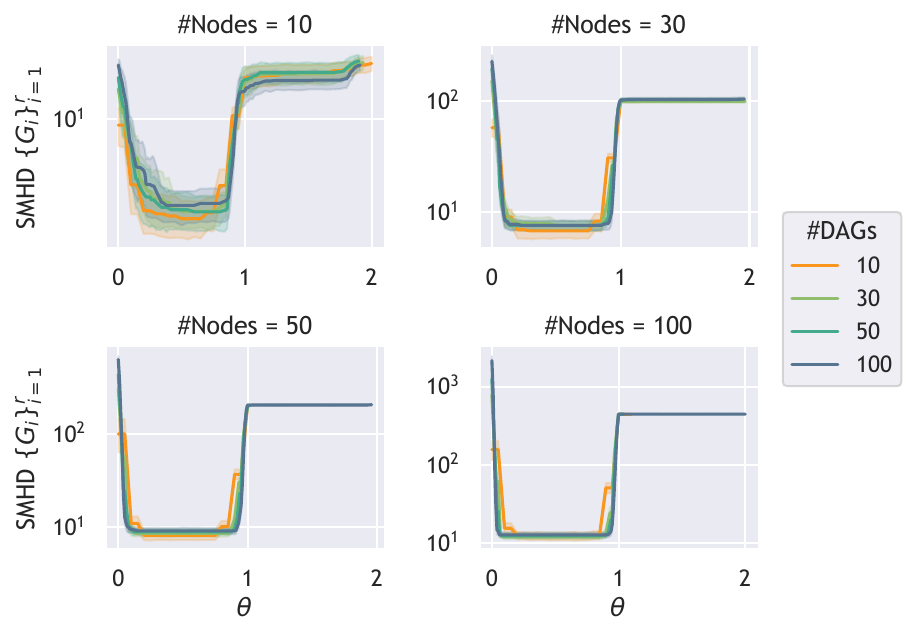}
        \caption{SMHD to input DAGs. Lower is better.}
        \label{fig:exp1-SMHD-input}
    \end{subfigure} \hfill
    \begin{subfigure}[b]{0.49\textwidth}
        \centering
        \includegraphics[width=\linewidth]{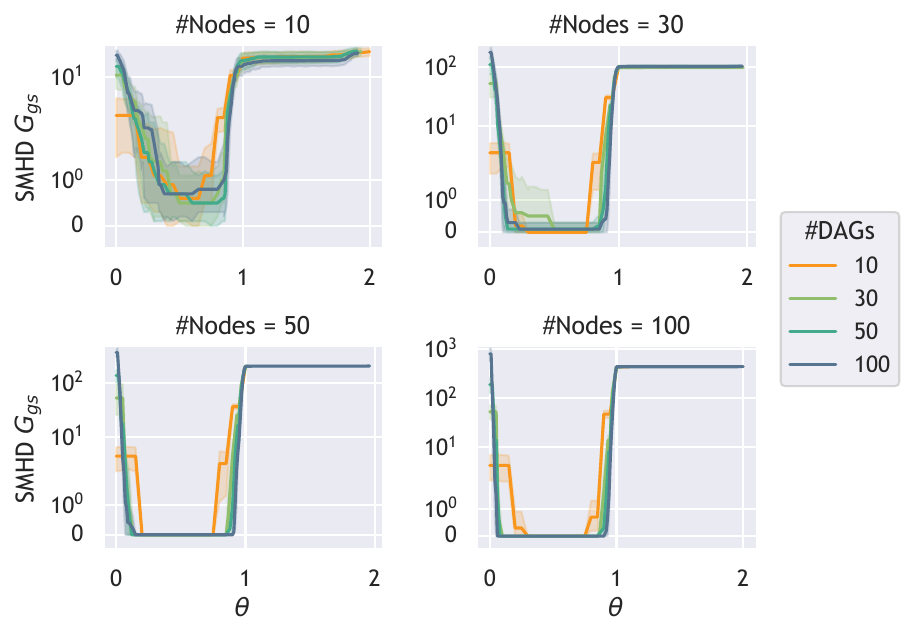}
        \caption{SMHD to gold-standard DAG $G_0$.}
        \label{fig:exp1-SMHD-gold-standard}
    \end{subfigure}
    \caption{Mean SMHD for synthetic experiments across different values of $\theta$.}
    \label{fig:exp1-SMHD-comparison}
\end{figure}

Figures~\ref{fig:exp1-SMHD-input} and \ref{fig:exp1-SMHD-gold-standard} report Structural Moral Hamming Distance (SMHD) between the consensus DAG $G^*$ and, respectively, the input DAGs $\{G_1, \dots, G_r\}$ and the ground truth $G_0$. As expected, increasing the number of input DAGs and nodes results in denser unrestricted fusion networks ($G^+$ at $\theta = 0$), which diverge from individual inputs. In contrast, MCBNC consistently yields compact and stable consensus structures. SMHD values drop rapidly for small $\theta$ (e.g., $\theta < 0.25$) and remain low across a wide interval. Conversely, as $\theta$ increases and relevant edges begin to be pruned (e.g., $\theta > 0.75$), the SMHD rises sharply.

Figure~\ref{fig:exp1-SMHD-gold-standard} shows that MCBNC also recovers the original structure $G_0$ with high fidelity, despite only observing perturbed inputs. Except for the smallest case ($n = 10$), perfect recovery (SMHD = 0) is typically achieved for $\theta \in [0.25, 0.75]$, indicating that $\theta = 0.5$ is a reasonable default in this setting.

The close alignment between the two SMHD curves confirms that optimizing for structural agreement with the input networks also improves the recovery of the actual underlying structure. These results validate the method’s behaviour under controlled conditions and support the findings reported in the main paper for the federated learning setting.

\FloatBarrier
\clearpage
%
%
\section{Extended Federated Learning Experimental Results} \label{sec:appendix_fl}
This section presents extended results for the federated learning experiments described in Section \textbf{Experimental Results}. These analyses cover a broader range of $r$ values, include additional structural metrics, and report on networks omitted from the main paper.

\subsection{Additional Metrics}\label{subsec:appx-additional}
We report three additional properties of the fused networks: SMHD to the input DAGs, edge count, and complete execution time.

\begin{figure*}[htb]
    \centering
    \includegraphics[width=1\linewidth, trim=0 0 0 0]{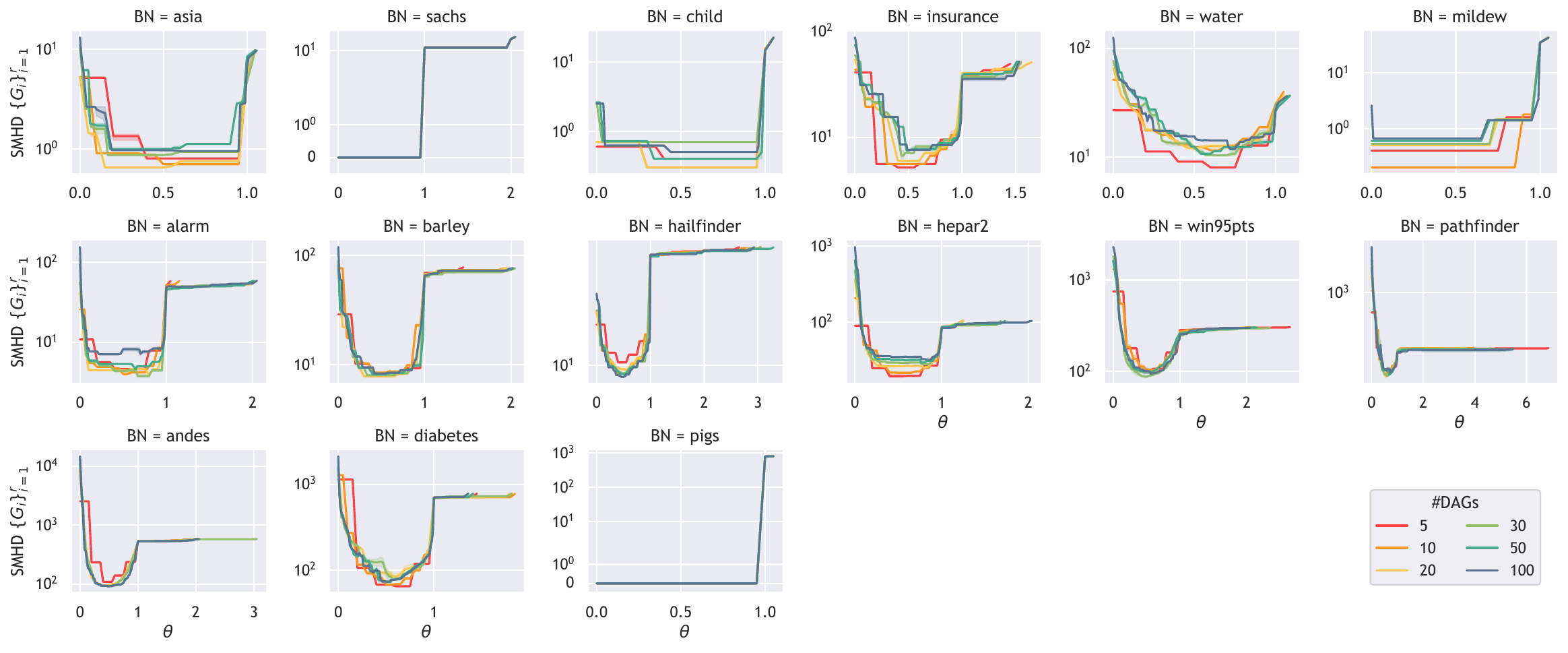}
    \caption{Mean SMHD between the consensus DAGs and the input DAGs, across pruning thresholds $\theta$. Lower is better.}
    \label{fig:exp3-SMHD-input}
\end{figure*}

\paragraph{Structural Agreement with Input DAGs.}
Figure~\ref{fig:exp3-SMHD-input} shows the mean SMHD between each consensus DAG and the input DAGs $\{G_i\}_{i=1}^r$, as a function of the pruning threshold $\theta$. This metric reflects structural consensus across participants and complements the SMHD-to-gold-standard curves in the main paper. As expected, the SMHD to the inputs increases with pruning. However, the minimum often coincides with the same $\theta$ that yields the best performance against the gold standard (see Fig. 1 in the main paper), validating the use of this metric for threshold selection when no reference model is available.

\begin{figure*}[htb]
    \centering
    \includegraphics[width=1\linewidth, trim=0 0 0 0]{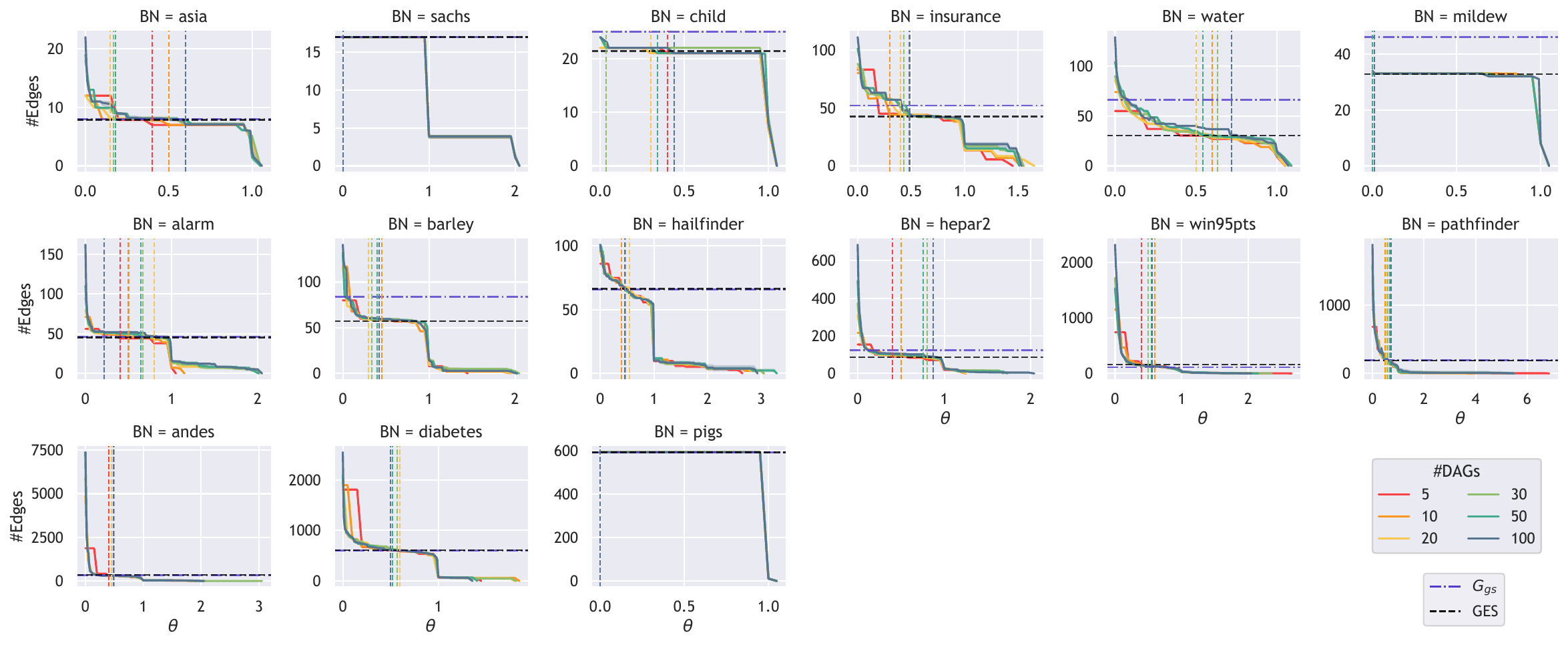}
    \caption{Mean edges across pruning thresholds $\theta$ for each BN. Dashed lines: selected~$\theta$ for each $\#$DAGs based on SMHD w.r.t. input BNs. Horizontal lines: average of input BNs from GES (black) and gold-standard BN (purple).}
    \label{fig:exp3-Edges}
\end{figure*}

\paragraph{Edge Count.}
Figure~\ref{fig:exp3-Edges} shows the number of edges in the consensus BNs for different values of $\theta$. The trends closely mirror those observed for treewidth in the main paper. At the empirically selected pruning thresholds, the number of edges in the fused networks aligns closely with that of both the GES-generated and gold-standard BNs. This confirms that MCBNC avoids the excessive overconnection typical of unrestricted fusion, while retaining the key structural features of the original.

\begin{figure*}[htb]
    \centering
    \includegraphics[width=1\linewidth, trim=0 0 0 0]{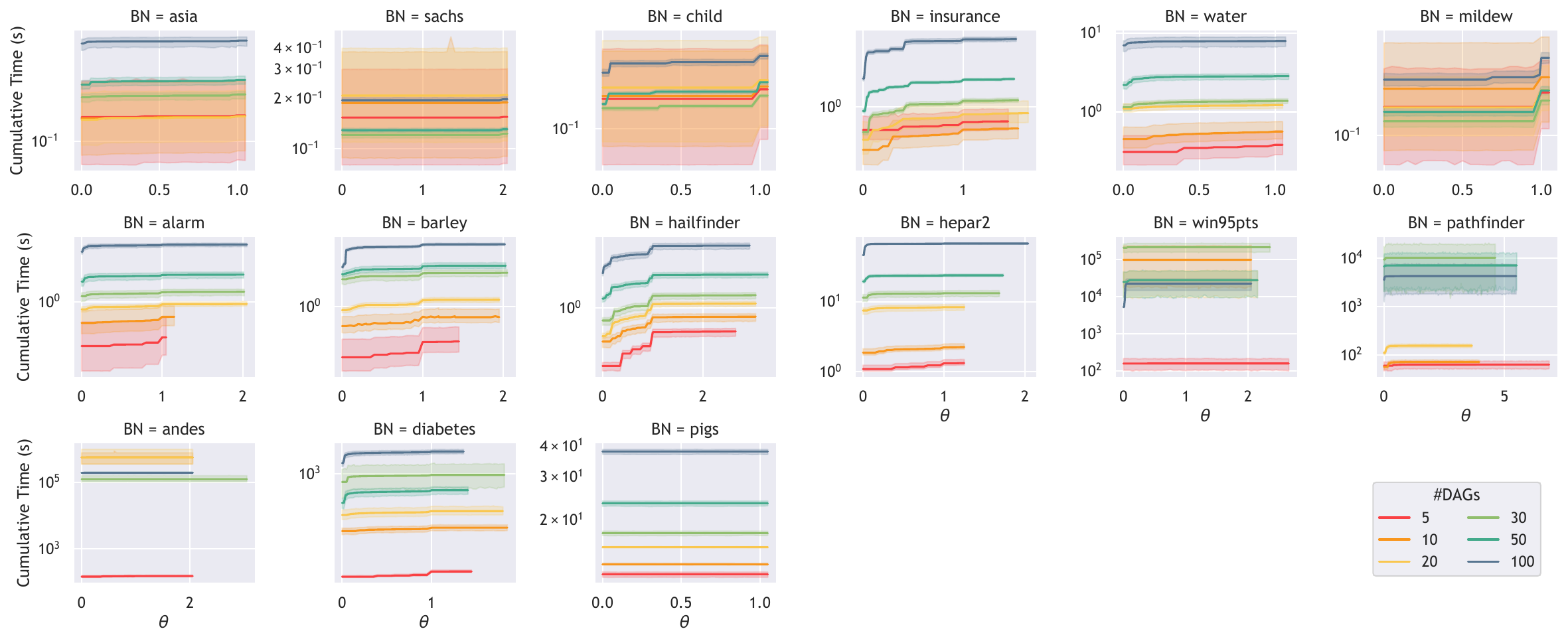}
    \caption{Cumulative execution time of MCBNC (in seconds) as a function of $\theta$.}
    \label{fig:exp3-Time}
\end{figure*}

\paragraph{Execution Time.}
Figure~\ref{fig:exp3-Time} shows the cumulative runtime of MCBNC (in seconds) for different values of $\theta$. Runtime increases with both the number and size of the input DAGs. In large or complex networks such as \textsc{Andes}, \textsc{Win95pts}, and \textsc{Pathfinder}, runtime is dominated by costly min-cut evaluations involving larger conditioning sets or intricate structures. Notably, most of the execution time is concentrated in early iterations (i.e., low $\theta$), when most pruning decisions are made. Variability is higher when runtimes are small (around one second or less), but tends to stabilize as runtimes increase. As $\theta$ grows, fewer edges are eligible for removal, and each iteration becomes progressively cheaper.

\FloatBarrier
\clearpage
\subsection{Threshold Selection Curves}\label{subsec:appx-threshold}
In Section~\textbf{Choosing the Fusion Threshold~$\theta$} of the main paper, we propose selecting $\theta$ by minimizing the Structural Moral Hamming Distance (SMHD) to the input DAGs. This proxy is computable without access to data or a gold standard, and the statistical analysis in the main paper shows that it yields structurally and statistically reliable results.

This appendix provides further justification for that strategy by reproducing the full SMHD–BDeu trade-off curves for all client counts $r \in \{5, 10, 20, 30, 50, 100\}$. Each plot shows how SMHD (to both the gold standard and the inputs) and normalized BDeu evolve as functions of the pruning threshold $\theta$. The BDeu scores are computed on a shared test dataset and averaged per benchmark.

Across all values of $r$ (Figures~\ref{fig:exp3-SMHD-BDeu-5DAGs}–\ref{fig:exp3-SMHD-BDeu-100DAGs}), we observe the same pattern: the value of $\theta$ that minimizes SMHD to the input DAGs almost always (i) minimizes or nearly minimizes SMHD to the gold-standard DAG, and (ii) achieves near-optimal BDeu scores. That is, selecting $\theta$ solely based on structural agreement with the inputs leads to a consensus structure that generalizes well, both structurally and in terms of likelihood. This validates the proposed threshold selection rule: a simple structural criterion, evaluated post hoc on the input graphs, suffices to guide model selection.

\begin{figure*}[htb]
    \centering
    \includegraphics[width=1\linewidth, trim=0 0 0 0]{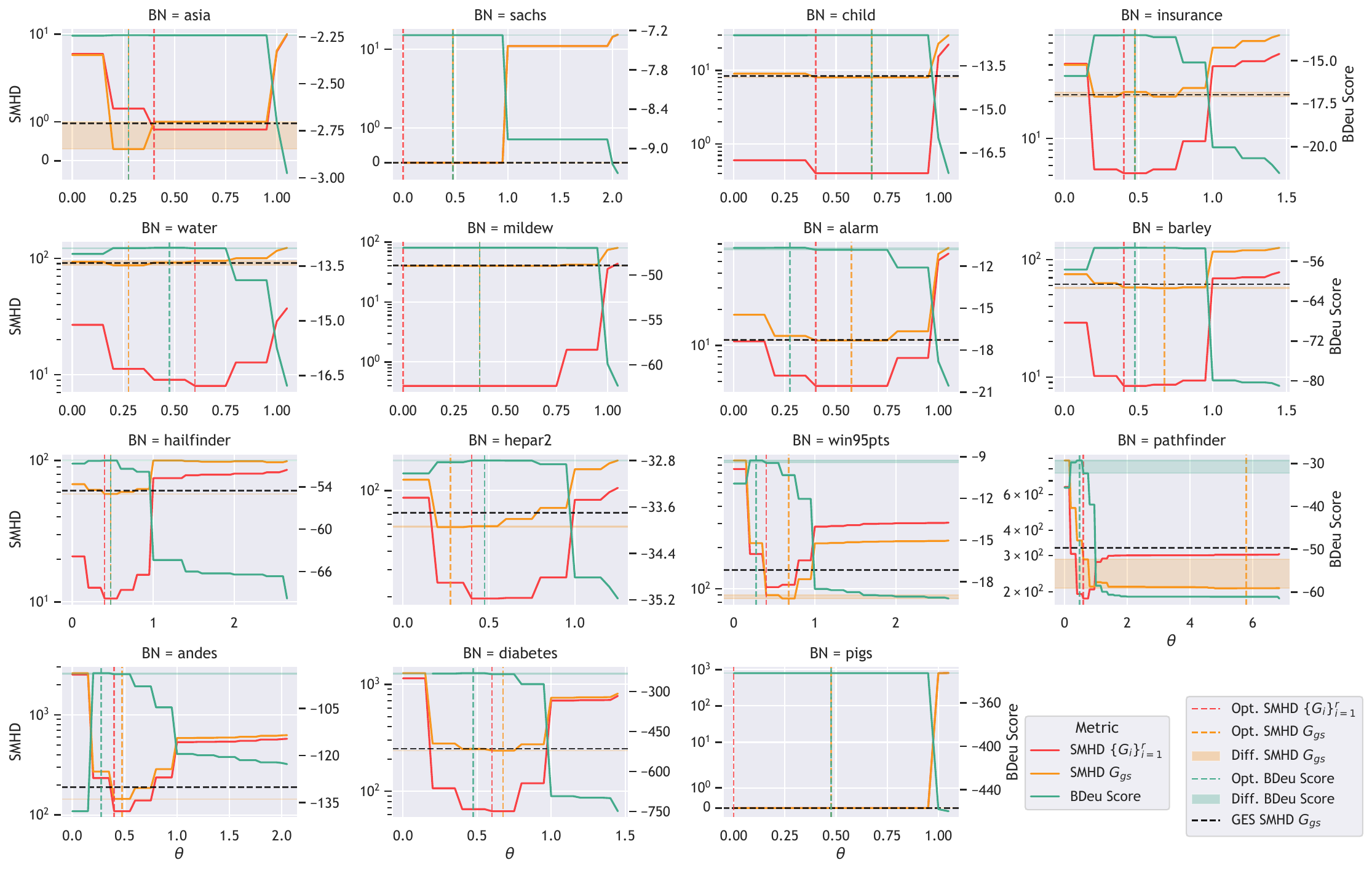}
    \caption{SMHD relative to GES-generated and gold standard BNs (left scale) and normalized BDeu score (right scale), using 5 DAGs.}
    \label{fig:exp3-SMHD-BDeu-5DAGs}
\end{figure*}

\begin{figure*}[htb]
    \centering
    \includegraphics[width=1\linewidth, trim=0 0 0 0]{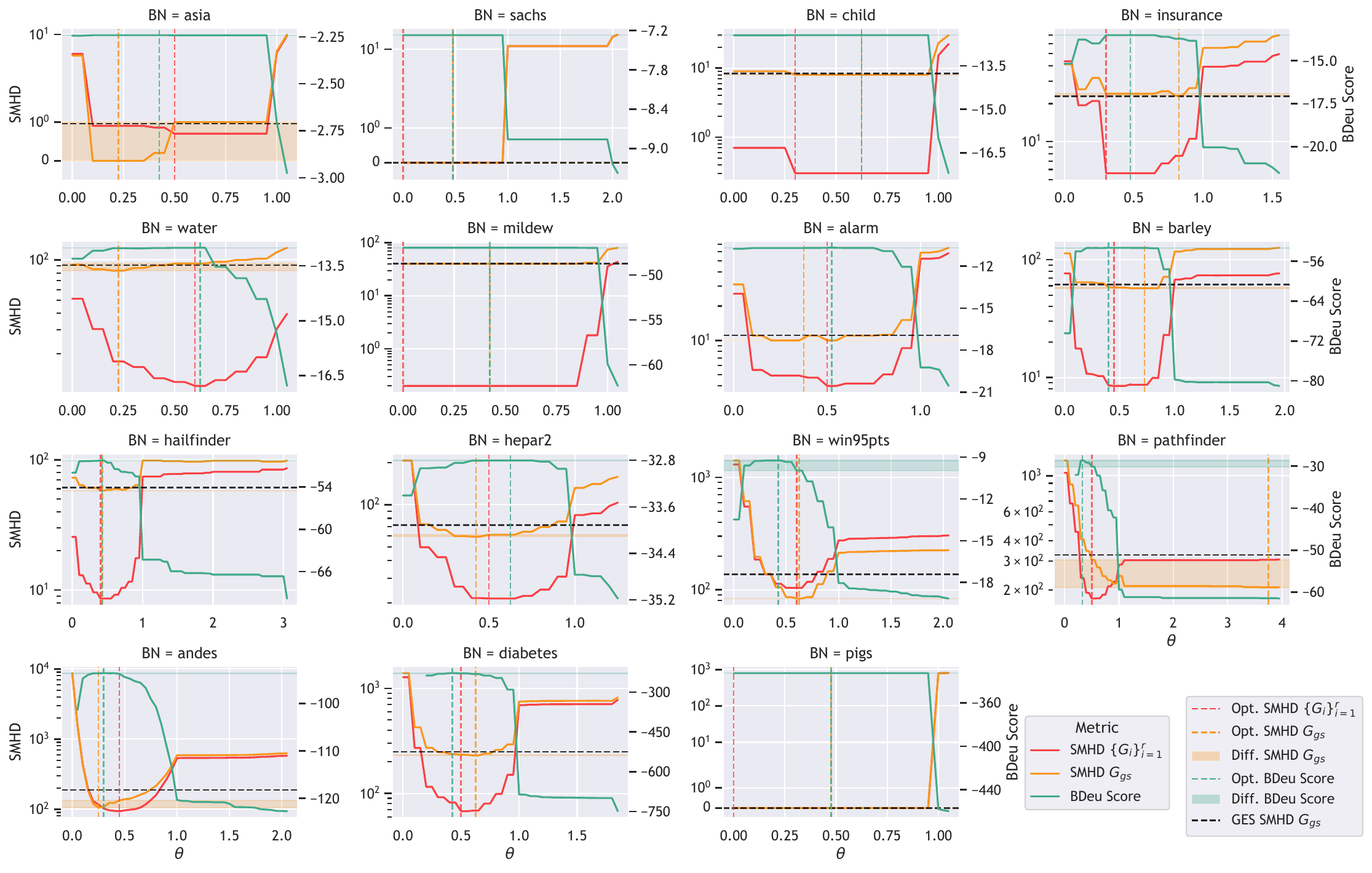}
    \caption{SMHD relative to GES-generated and gold standard BNs (left scale) and normalized BDeu score (right scale), using 10 DAGs.}
    \label{fig:exp3-SMHD-BDeu-10DAGs}
\end{figure*}

\begin{figure*}[htb]
    \centering
    \includegraphics[width=1\linewidth, trim=0 0 0 0]{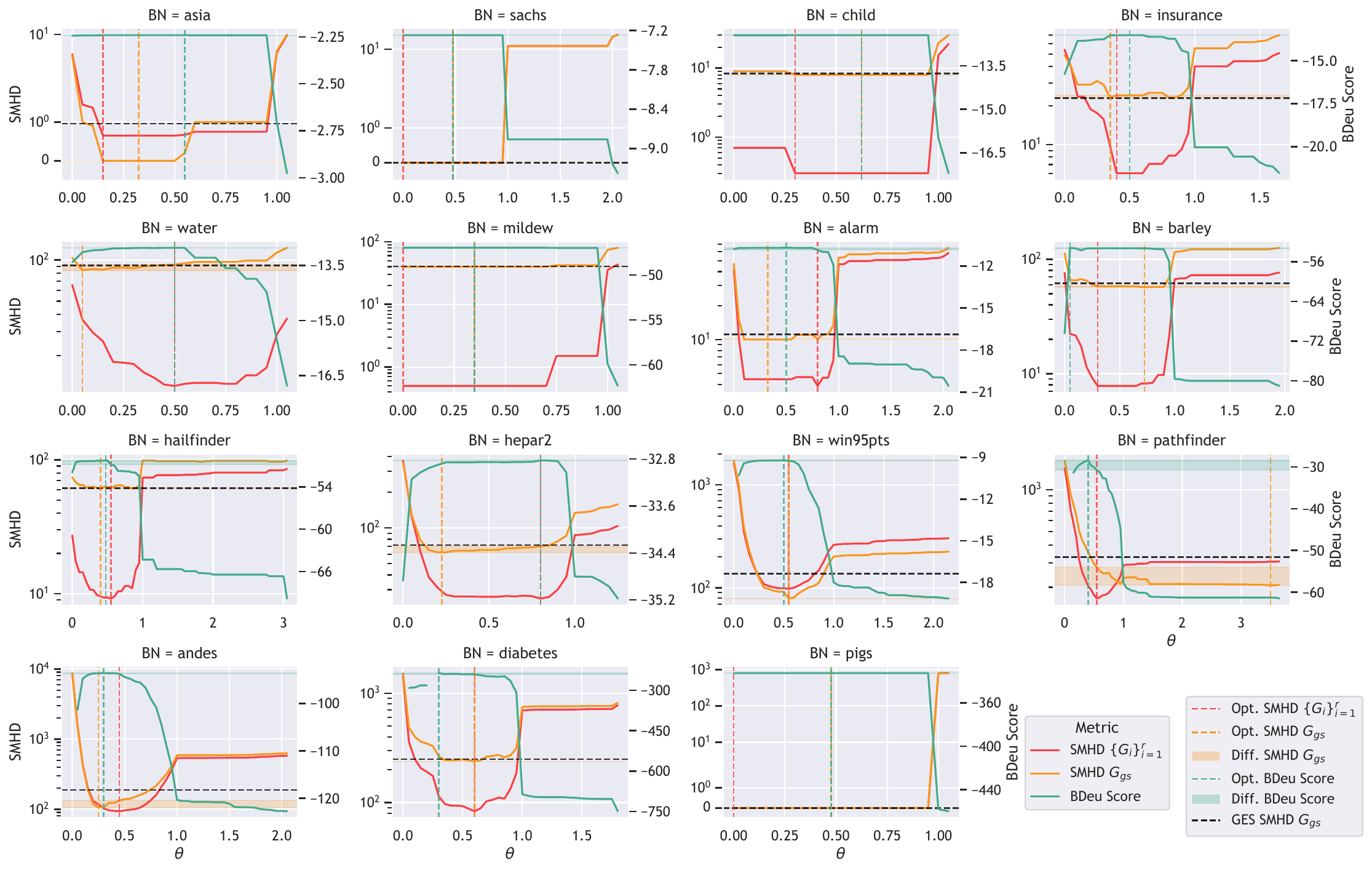}
    \caption{SMHD relative to GES-generated and gold standard BNs (left scale) and normalized BDeu score (right scale), using 20 DAGs.}
    \label{fig:exp3-SMHD-BDeu-20DAGs}
\end{figure*}

\begin{figure*}[htb]
    \centering
    \includegraphics[width=1\linewidth, trim=0 0 0 0]{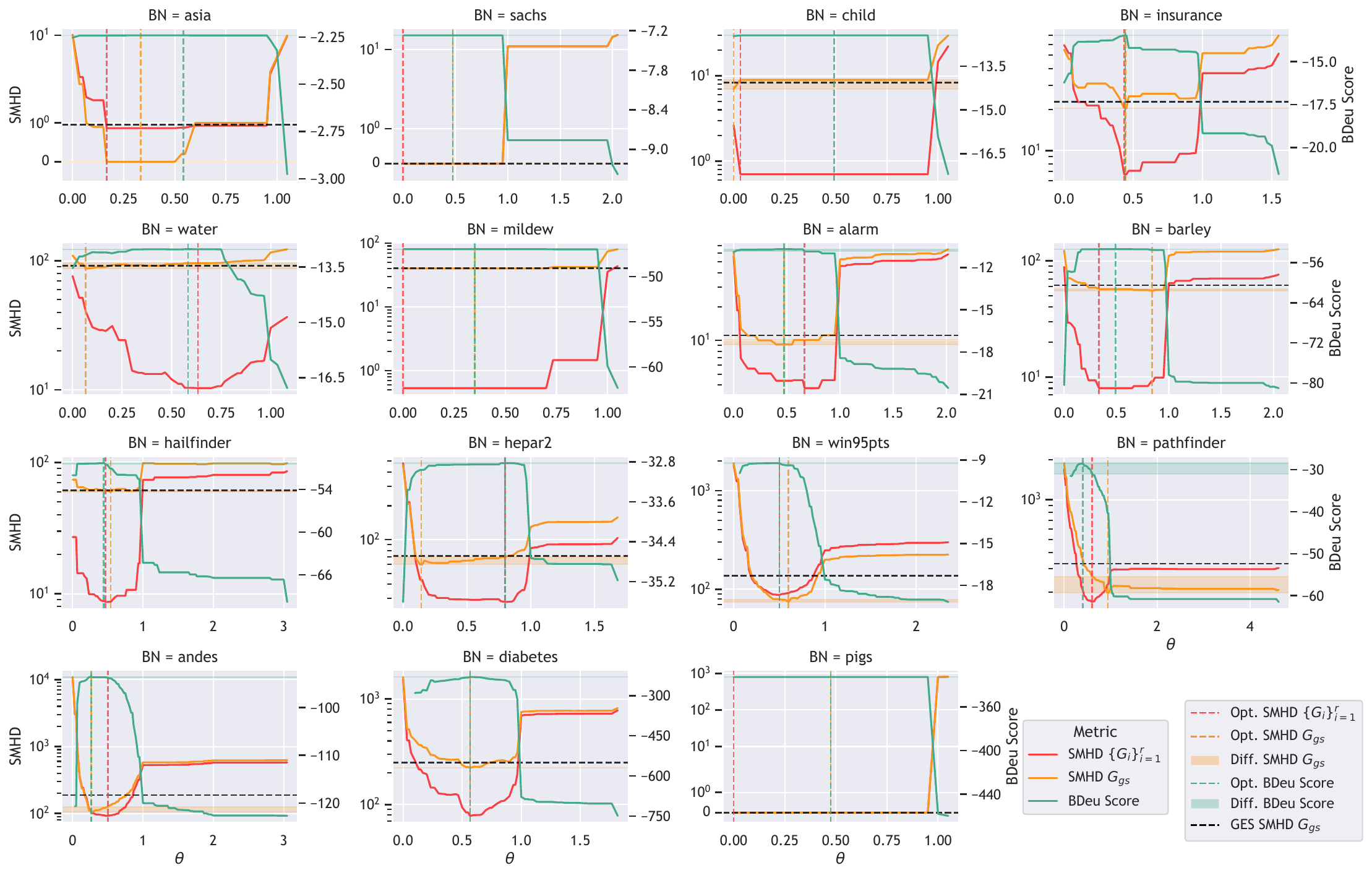}
    \caption{SMHD relative to GES-generated and gold standard BNs (left scale) and normalized BDeu score (right scale), using 30 DAGs.}
    \label{fig:exp3-SMHD-BDeu-30DAGs}
\end{figure*}

\begin{figure*}[htb]
    \centering
    \includegraphics[width=1\linewidth, trim=0 0 0 0]{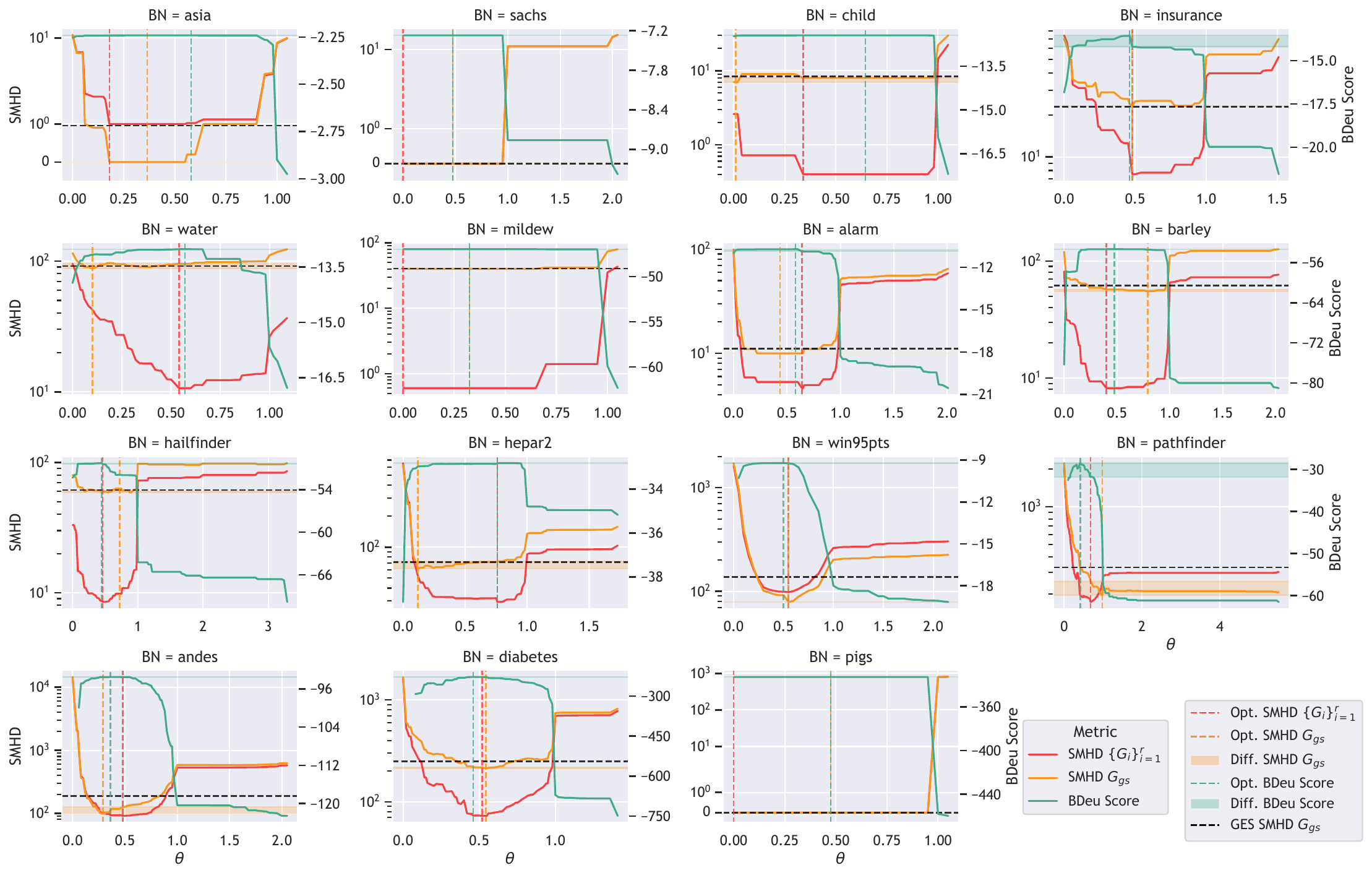}
    \caption{SMHD relative to GES-generated and gold standard BNs (left scale) and normalized BDeu score (right scale), using 50 DAGs.}
    \label{fig:exp3-SMHD-BDeu-50DAGs}
\end{figure*}

\begin{figure*}[htb]
    \centering
    \includegraphics[width=1\linewidth, trim=0 0 0 0]{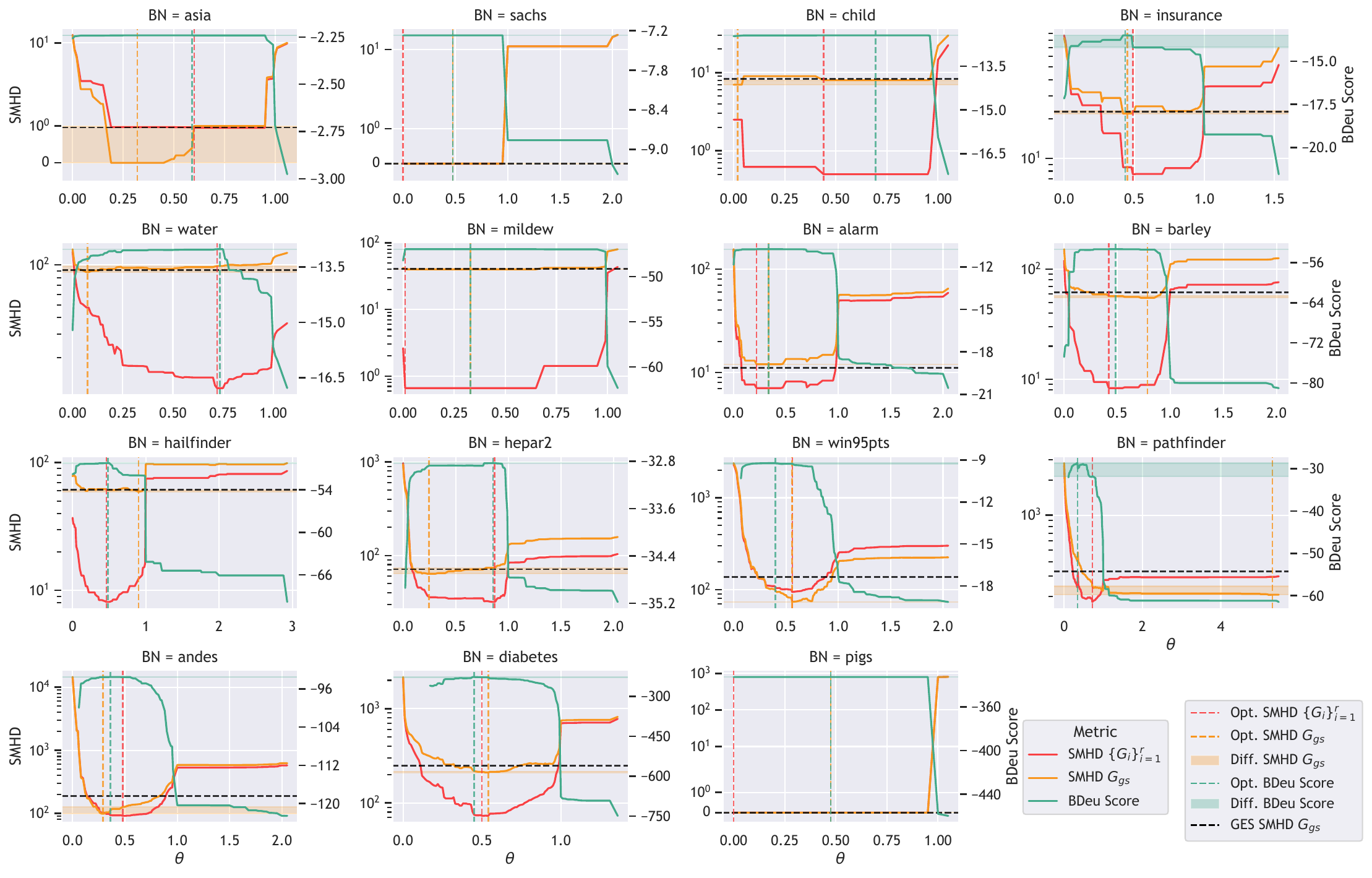}
    \caption{SMHD relative to GES-generated and gold standard BNs (left scale) and normalized BDeu score (right scale), using 100 DAGs.}
    \label{fig:exp3-SMHD-BDeu-100DAGs}
\end{figure*}

\clearpage
\FloatBarrier
\subsection{Results Including \textsc{Sachs} and \textsc{Pigs} BNs}\label{subsec:appx-sachs}
The main paper explains that the \textsc{Sachs} and \textsc{Pigs} networks were excluded from the core analysis because GES consistently reconstructs their gold-standard DAGs. Consequently, the initial fusion $G^+$ already matches the target structure, and MCBNC performs no pruning until $\theta \geq 1$. Since no improvement is possible, these cases offer little insight into the algorithm’s behaviour when structural disagreement exists.

For completeness, Figures~\ref{fig:exp3-SMHD-gold-standard-appx} through~\ref{fig:exp3-TW-appx} reproduce the evaluation curves with these two networks included. As expected, all metrics (SMHD, BDeu, and treewidth) remain flat throughout the entire trajectory, until unnecessary pruning begins at $\theta = 1$. This confirms that MCBNC preserves an optimal consensus when the inputs already match the gold standard.

\begin{figure*}[htb]
    \centering
    \includegraphics[width=1\linewidth, trim=0 0 0 0]{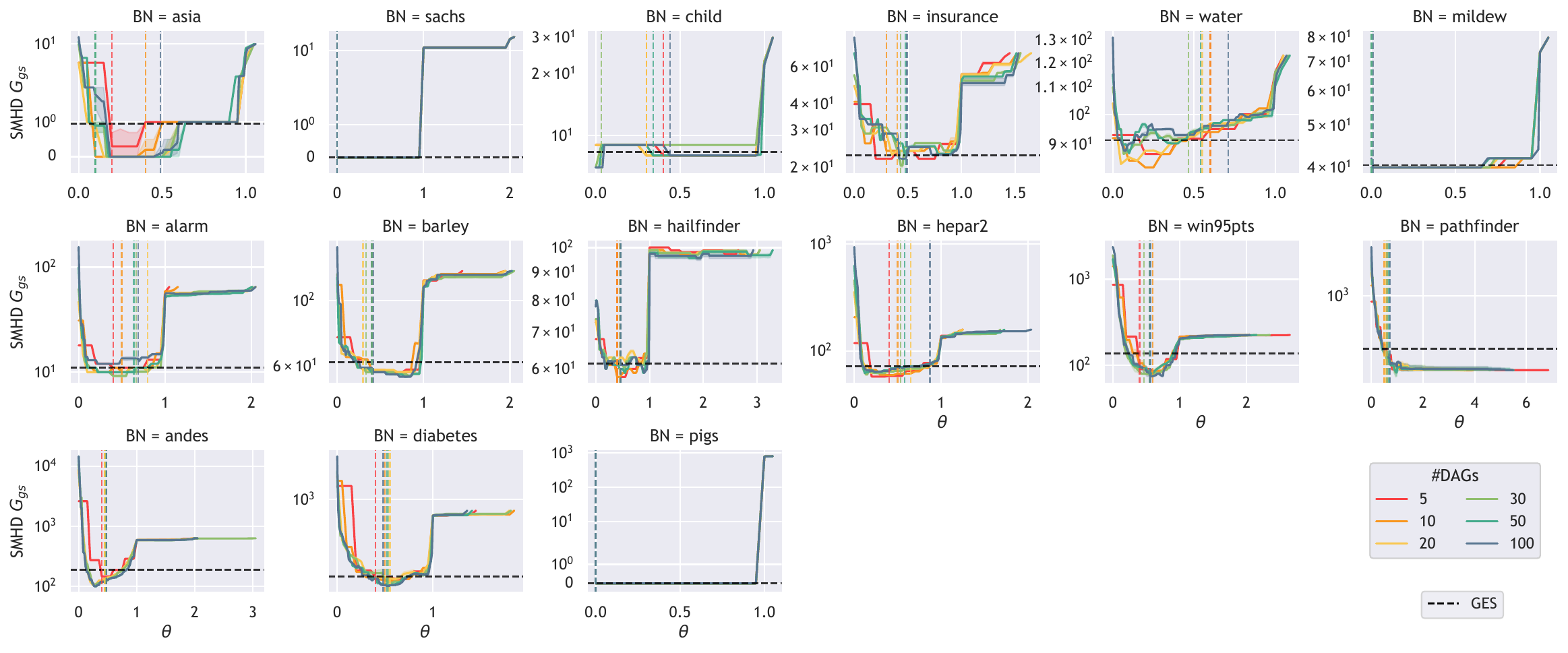}
    \caption{Mean SMHD to the gold-standard DAG across thresholds $\theta$ for each BN. Leftmost point: full fusion $G^+$. Rightmost: empty DAG $\emptyset$. Horizontal line: average SMHD of GES-generated input DAGs. Lower is better.}
    \label{fig:exp3-SMHD-gold-standard-appx}
\end{figure*}

\begin{figure*}[htb]
    \centering
    \includegraphics[width=1\linewidth, trim=0 0 0 0]{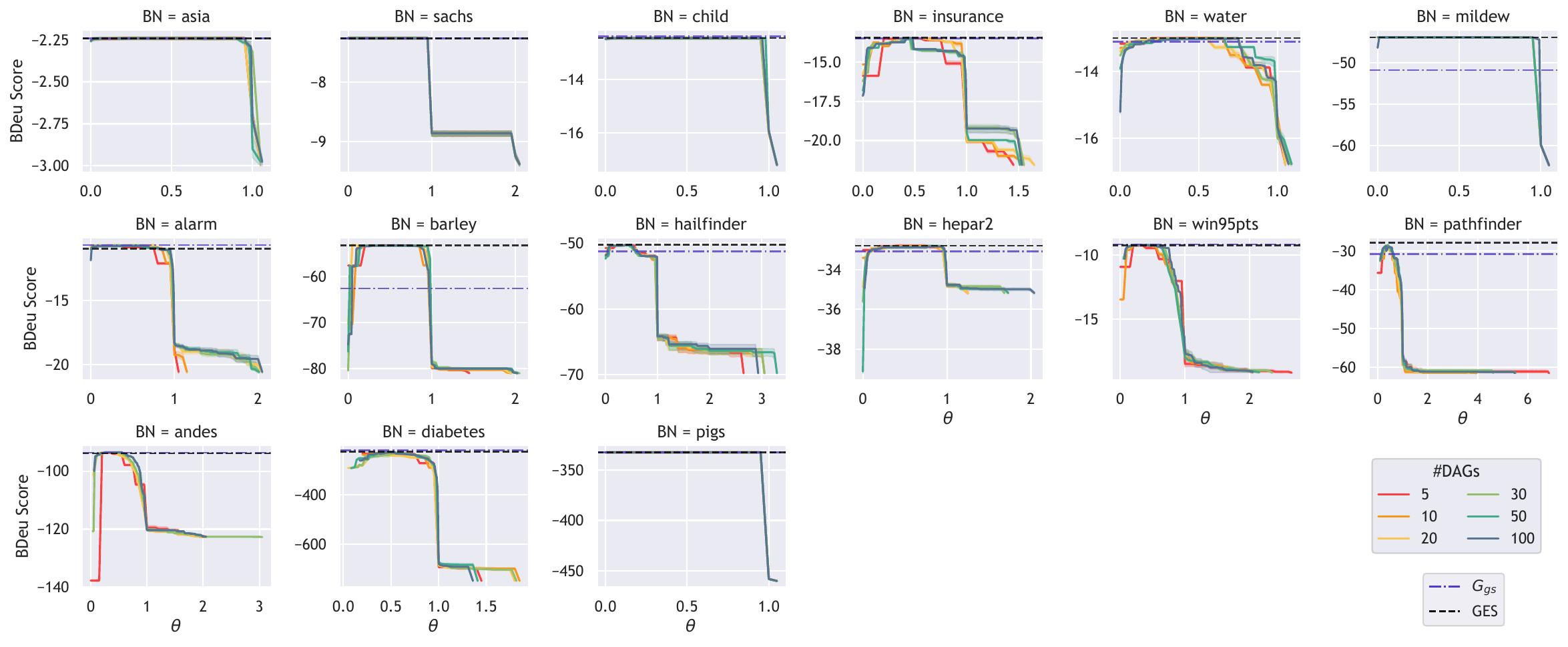}
    \caption{Mean BDeu score across thresholds $\theta$ for each BN. Leftmost point: full fusion $G^+$. Rightmost: empty DAG $\emptyset$. Horizontal lines: average BDeu of GES input DAGs (black) and gold-standard DAG (purple). Higher is better.}
    \label{fig:exp3-BDeu-appx}
\end{figure*}

\begin{figure*}[htb]
    \centering
    \includegraphics[width=1\linewidth, trim=0 0 0 0]{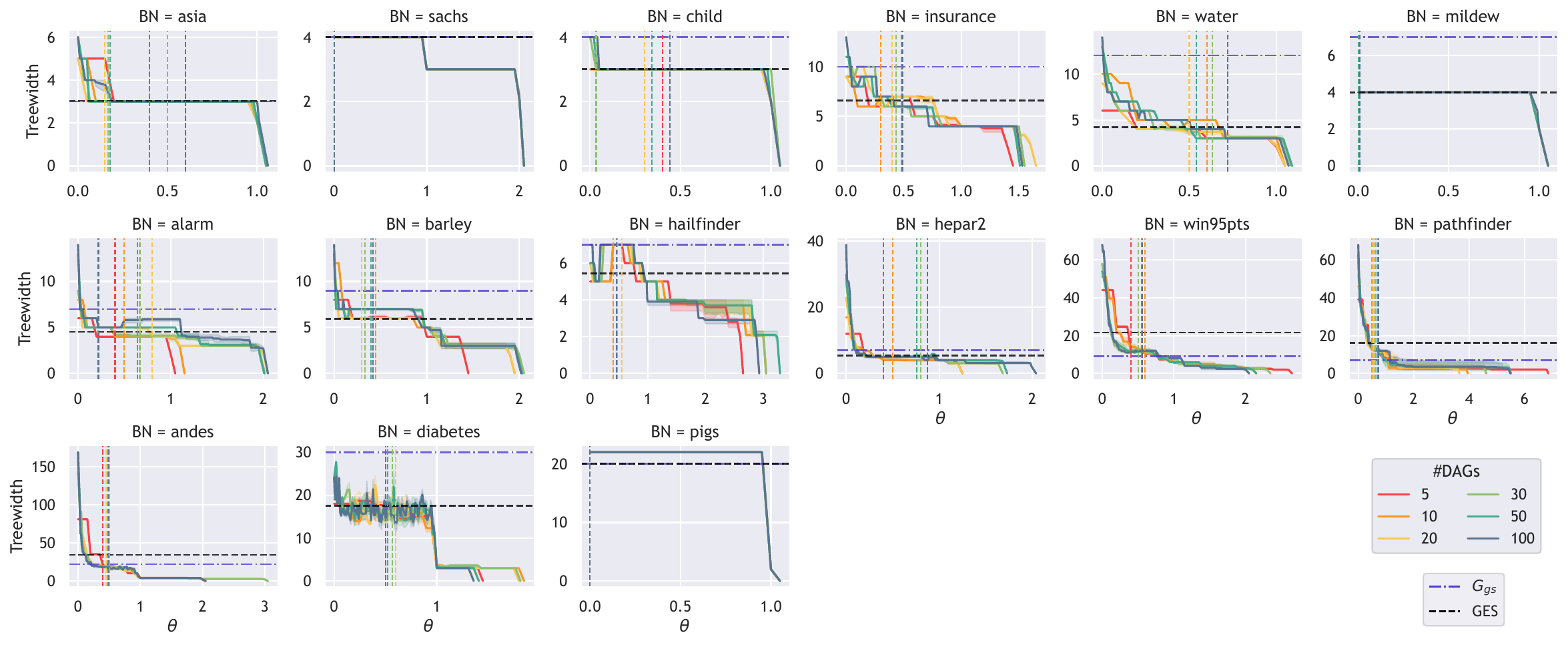}
    \caption{Treewidth of the consensus DAG across thresholds $\theta$ for each BN. Leftmost point: full fusion $G^+$. Rightmost: empty DAG $\emptyset$. Horizontal lines: average treewidth of GES input DAGs (black) and gold-standard DAG (purple). Lower is better.}
    \label{fig:exp3-TW-appx}
\end{figure*}

\FloatBarrier
\subsection{Sensitivity to the Conditioning-Set Size $k_{\max}$}\label{subsec:appx-kmax}
We assess the sensitivity of \textsc{MCBNC} to the conditioning-set cap $k_{\max}$ on the largest tested BN, \textsc{Diabetes} ($n=413$). For each $k_{\max} \in \{0,1,2,3,4,5,10,15,20\}$ and each $r \in \{5,10,20,30,50,100\}$ input DAGs, we ran the full pruning routine and recorded two metrics: (i) structural accuracy to the gold-standard DAG $G_{\text{gs}}$, and (ii) total wall-clock time.

Figure~\ref{fig:kmax-smhd-gs} shows the SMHD compared to the gold standard for all values of $k_{\max}$ and $r$. The curves are visually very similar, indicating that pruning quality is largely unaffected by the cap. The most visible differences occur at $r=30$, where specially $k_{\max}=0$ and $k_{\max}=20$ deviate slightly. However, these variations are not systematic and likely stem from randomness and the effect of a greedy search rather than from $k_{\max}$ itself, particularly for large values, which only expand the search space.

\begin{figure}[htb]
  \centering
  \includegraphics[width=\linewidth]{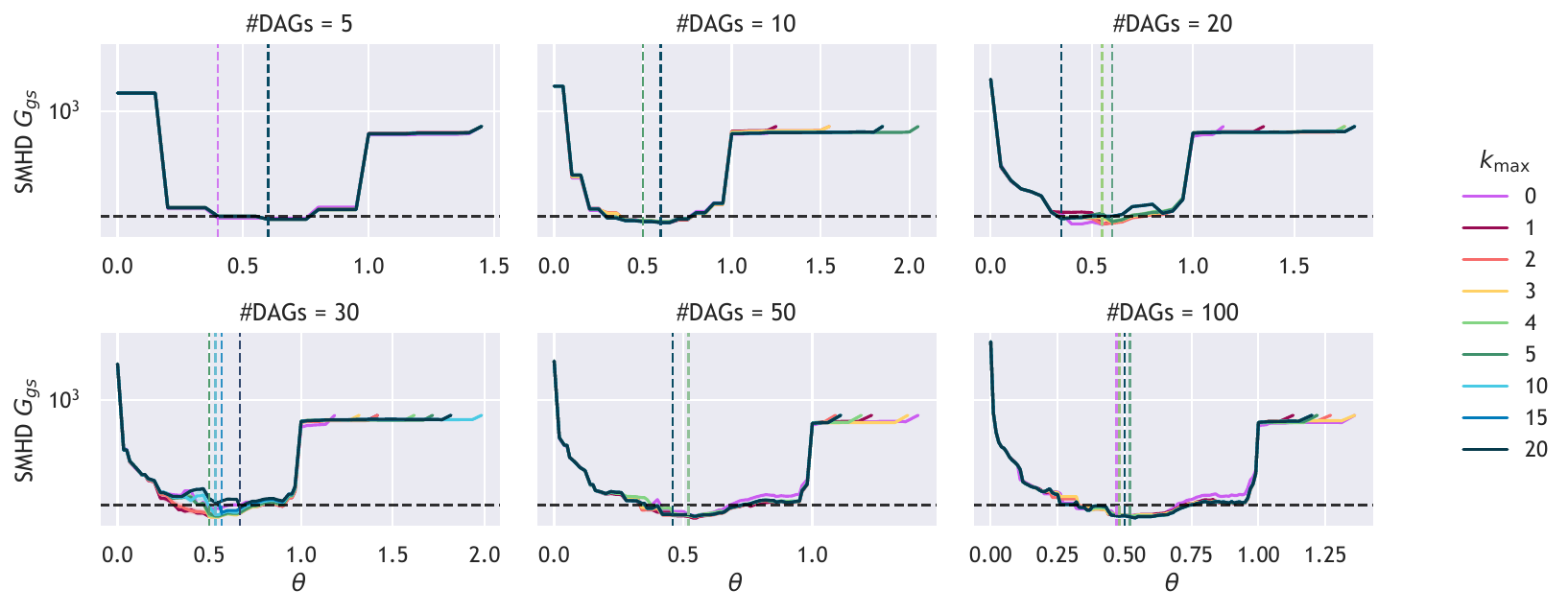}
  \caption{SMHD to the gold-standard DAG $G_{\text{gs}}$ for varying $k_{\max}$, pruning threshold $\theta$, and number of input DAGs $r$ on the \textsc{Diabetes} BN. The horizontal dotted line marks the average SMHD to $G_{\text{gs}}$ of the input DAGs.}

  \label{fig:kmax-smhd-gs}
\end{figure}

Figure~\ref{fig:kmax-bars} provides a complementary summary: it shows the final SMHD to the gold standard at the optimal pruning threshold $\theta$ (corresponding to the vertical lines in Figure~\ref{fig:kmax-smhd-gs}) for each combination of $k_{\max}$ and $r$. All configurations with $k_{\max} \in \{2,3,4,5,10,15,20\}$ achieve nearly identical accuracy. The only pronounced deviations occur at $r = 30$ for $k_{\max} = 0$ and $k_{\max} = 20$, consistent with the earlier curves. Nevertheless, all results lie below the average SMHD of the individual GES input networks (represented by the horizontal dotted line, note the axis zoom). These findings confirm that the performance of \textsc{MCBNC} is not sensitive to the conditioning-set cap. Greedy tie-breaking introduces more structural variation than $k_{\max}$ itself. While increasing $k_{\max}$ enlarges the space of testable independencies, it does not lead to systematically better pruning, and mostly alters the deletion order among weakly supported edges.

\begin{figure}[htb]
  \centering
  \includegraphics[width=\linewidth]{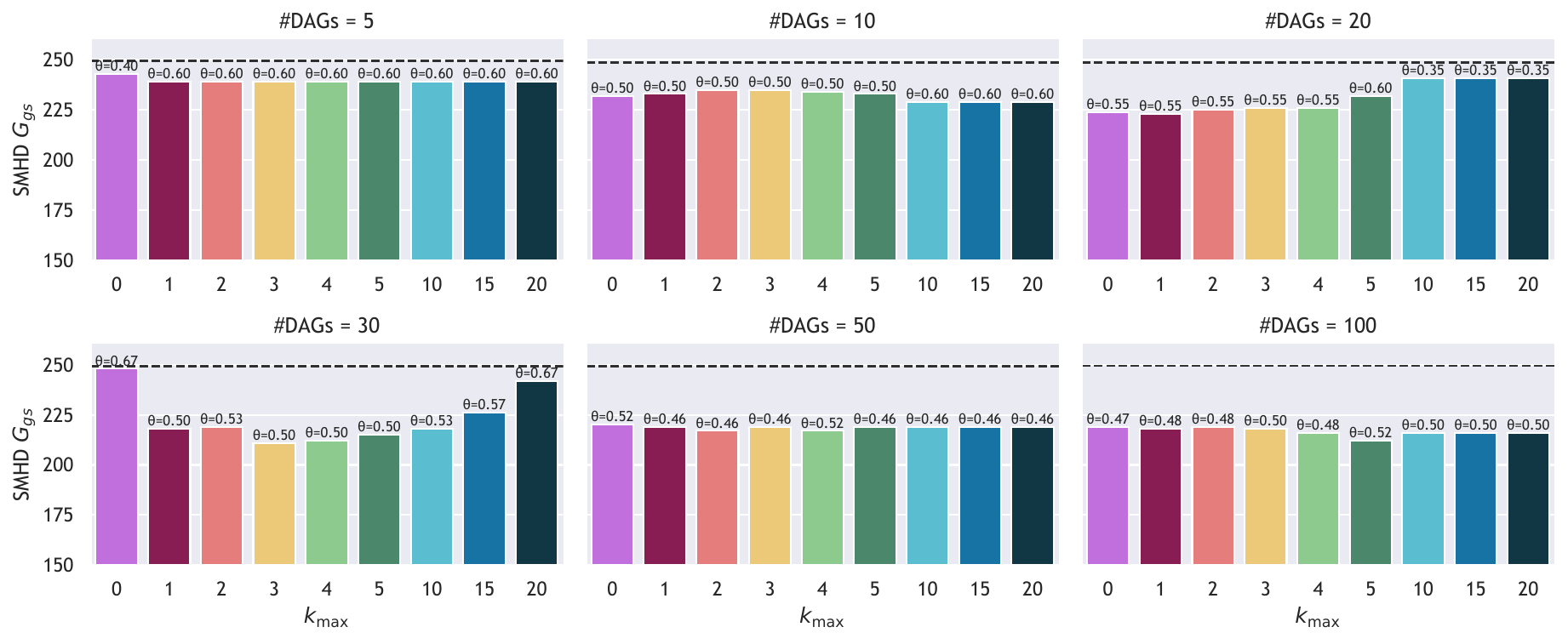}
  \caption{SMHD to the gold standard $G_{\text{gs}}$ at selected $\theta$ for each $k_{\max}$ for \textsc{Diabetes} BN. The horizontal dotted line marks the average SMHD to $G_{\text{gs}}$ of the input DAGs.}
  \label{fig:kmax-bars}
\end{figure}

Figure~\ref{fig:kmax-time} reports the cumulative runtime. As expected, execution time generally increases with $k_{\max}$ due to the exponential number of conditioning subsets. However, some deviations occur: for instance, at $r = 5$, the highest runtime corresponds to $k_{\max} = 0$. This value is an atypical setting in a BES-style search, which may introduce instability and redundant operations. Runtime also does not grow monotonically with $r$. For example, pruning takes longer at $r = 30$ than at $r = 50$ or even $r = 100$ for $k_{\max} = 15$ and $k_{\max} = 20$. This reflects the fact that pruning complexity depends not only on input size but also on the specific substructures generated during the fusion process. In particular, denser or more entangled intermediate CPDAGs can increase the cost of min-cut evaluations. 

These results confirm that, in practice, the exponential term $2^{k_{\max}}$ in the theoretical runtime bound $O(r,m^3,2^{k_{\max}})$ (see Lemma~3) has limited impact. Large conditioning sets are rarely generated, so the worst-case complexity is seldom reached. Nonetheless, when the number of input DAGs is high and the fused structure becomes densely connected, complex subgraphs can emerge, triggering expensive evaluations. This explains why pruning is sometimes slower for $r = 30$ than for $r = 50$ or $r = 100$, depending on the particular connectivity patterns formed during fusion.

\begin{figure}[htb]
  \centering
  \includegraphics[width=\linewidth]{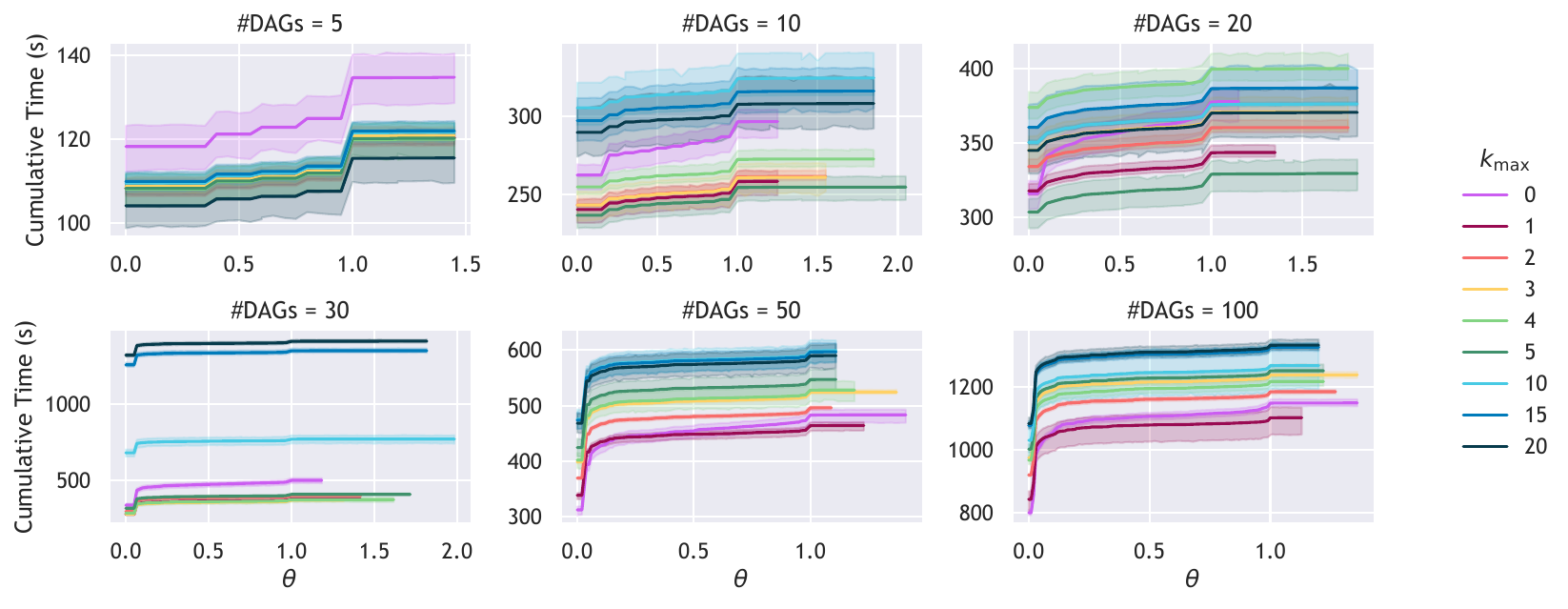}
  \caption{Cumulative runtime as a function of $k_{\max}$ and $\theta$ for \textsc{Diabetes} BN.}
  \label{fig:kmax-time}
\end{figure}

\clearpage
%
%
\section{Ford-Fulkerson Algorithm} \label{sec:ford-fulkerson-appx}
The Ford-Fulkerson algorithm \cite{Ford1956} computes the maximum flow $f^*$ in a network $D=(V,E)$ with capacity function $c: E \to \mathbb{R}^+$ by iteratively augmenting the flow along paths from the source $s$ to the sink $t$. Initially, the flow on every edge is set to zero, i.e., $f(e)=0$ for all $e\in E$. The residual graph $D_f=(V,E_f)$ is constructed as follows: for each edge $e=(u \to v) \in E$, include the forward edge $e$ in $E_f$ with residual capacity $r(u \to v)=c(u \to v)-f(u \to v)$, and also include the reverse edge $e' = (v \to u)$ with residual capacity $r(v \to u)=f(u \to v)$.

At each iteration, an augmenting path $p$ from $s$ to $t$ is identified in $D_f$ (commonly via a breadth-first search), and its bottleneck capacity is computed as $f_p=\min_{e\in p} r(e)$. Then, for every edge $e\in p$ with corresponding reverse edge $e'$, update the flow and residual capacities as follows: 
\begin{equation*}
    f(e) = f(e)+f_p, \quad r(e) = r(e)-f_p, \quad r(e') = r(e')+f_p.
\end{equation*}
This process repeats until no augmenting paths from $s$ to $t$ exist in $D_f$. At termination, the maximum flow is given by 
\begin{equation*}
f^* = \text{val}(f)=\sum_{e\in\delta^+(s)} f(e).
\end{equation*}
The final residual graph defines the minimum cut by partitioning $V$ into two disjoint sets: $S^*$, the set of vertices reachable from $s$ in $D_f$, and $T^* = V \setminus S^*$. The set of cut edges is 
\begin{equation*}
\{ (u \to v) \in E \mid u \in S^*,\, v \in T^*,\, r(u \to v)=0 \}.
\end{equation*}
By the Max-Flow Min-Cut Theorem \cite{Ahuja1993,Ford1956}, the total capacity of this cut equals $f^*$.

%
%
\section{Illustrative Example of MCBNC Algorithm} \label{sec:example}
To better illustrate the mechanics of MCBNC, this section walks through a complete worked example using three small DAGs defined over a shared set of variables. We demonstrate how the initial fusion is constructed, how edge criticality is computed via the min-cut algorithm, and how pruning decisions are made across iterations. The example illustrates the effect of the pruning threshold $\theta$ and demonstrates how consensus is progressively achieved.

\subsection{Initialization}

Consider three directed acyclic graphs (DAGs) $\{G_i\}_{i=1}^{3}$ defined over the variable set $V = \{w, x, y, z\}$, with corresponding edge sets:
\begin{align*}
    E_1 &= \{w \to x,\ x \to y,\ y \to z\}, \\
    E_2 &= \{w \to x,\ w \to y,\ x \to z\}, \\
    E_3 &= \{w \to x,\ y \to x,\ x \to z\}.
\end{align*}
A heuristic ordering $\sigma = (w, y, x, z)$ is obtained using the method proposed in \cite{Puerta2021Fusion}. The transformed DAGs\footnote{Note that $G_2$ and $G_3$ already comply with $\sigma$, i.e., $G_2 = G_2^{\sigma}$ and $G_3 = G_3^{\sigma}$, while $G_1 \neq G_1^{\sigma}$.} $\{G_i^{\sigma}\}_{i=1}^{3}$, obtained by aligning the edges to respect $\sigma$, have edge sets:
\begin{align*}
    E_1^{\sigma} &= \{w \to x,\ w \to y,\ y \to x,\ y \to z\}, \\
    E_2^{\sigma} &= \{w \to x,\ w \to y,\ x \to z\}, \\
    E_3^{\sigma} &= \{w \to x,\ y \to x,\ x \to z\}.
\end{align*}
The initial fused graph is obtained by taking the union of the transformed edge sets:
\begin{equation*}
G^+ = (V, E^+), \quad \text{where} \quad E^+ = E_1^{\sigma} \cup E_2^{\sigma} \cup E_3^{\sigma}.
\end{equation*}
Expanding $E^+$ explicitly,
\begin{equation*}
    E^+ = \{w \to x, w \to y, x \to z, y \to x, y \to z\}.
\end{equation*}
Before performing any min-cut analysis between two nodes, we extract the ancestral subgraph of the node pair and the conditioning set, and moralize only that subgraph. For this example, we set the threshold $\theta=0.5$, meaning that any edge with a criticality score $\Psi_e$ below this value will be pruned.

\subsection{First Iteration}

The algorithm iteratively evaluates each edge $e \in E^+$ by analyzing all possible conditioning sets $H \subseteq \mathcal{P}_e$ in the actual iteration. For each $H$, we first extract the ancestral subgraph of the nodes $\{u, v\} \cup H$ from each input DAG $\{G_i\}_{i=1}^{3}$, then moralize this subgraph to produce $\{\widetilde{G}_i\}_{i=1}^{3}$, and finally remove the conditioning set $H$ to construct the conditioned graphs $\{\widetilde{G}_i^{H}\}_{i=1}^{3}$. The size of these conditioning sets is limited by a parameter $k_{\max}$ to ensure computational tractability. In this example, all arcs are directed during the first iteration, and $H = \emptyset$ for every edge, as no valid conditioning sets exist yet. Subsequent iterations may consider non-empty conditioning sets as the network structure evolves.

For each edge $e = (u \to v) \in E^+$, the criticality score is computed as:
\begin{equation*}
    \Psi^H_{(u\to v)} = \frac{1}{3} \sum_{i=1}^{3} \left| S_i^H \right|,
\end{equation*}
where $S_i^H$ is the min-cut set in $\widetilde{G}_i^H$. Evaluating $\Psi_e$ for each edge:

\[
\begin{array}{lllll}
    \Psi^{\{\}}_{(w\to x)} = 1.0, & \Psi^{\{\}}_{(y\to z)} = 0.\wideparen{3}, & \Psi^{\{\}}_{(w\to y)} = 0.\wideparen{6}, & \Psi^{\{\}}_{(x\to z)} = 0.\wideparen{6}, & \Psi^{\{\}}_{(y\to x)} = 0.\wideparen{6}.
\end{array}
\]

Since the minimal score $\Psi_{(y\to z)}^{\{\}} = 0.\wideparen{3} < \theta = 0.5$, the edge $(y\to z)$ is removed from $E^+$ with empty conditioning set ($\{\}$) using Chickering's operator \cite{chickering_optimal_2002}, yielding:
\begin{equation*}
    G^+ = \left(V, E^+\right), \quad E^+ = \left\{w \to x, w \to y, x \to z, y \to x \right\}.
\end{equation*}

Additionally, $(y \to z)$ is removed from the original DAGs, updating $G_1$ to
\begin{equation*}
    G_1 = (V, E_1 = \{w \to x, x \to y\}).
\end{equation*}

The fused DAG $G^+$ is then converted into a CPDAG, yielding the result of the first iteration:
\begin{equation*}
    G^*_{(1)} = \left(V, E^*_{(1)}\right), \quad E^*_{(1)} = \left\{w - x, w - y, x - z, y - x \right\}.
\end{equation*}

\subsection{Second Iteration}  
In the second iteration, we recompute the min-cut values for the fused edges obtained in the previous iteration $G^*_{(1)}$. For undirected edges, both orientations are evaluated separately. For instance, the edge $e = (w - x)$ yields the arcs 
\begin{equation*}
    e^{\rightarrow} = (w \rightarrow x) \quad \text{and} \quad e^{\leftarrow} = (w \leftarrow x).
\end{equation*}
Following the same procedure as in the first iteration, we compute the criticality score $\Psi^{H}_{(u \to v)}$ for each arc $e = (u \to v) \in E^+$ and each of its conditioning sets $H \subseteq \mathcal{P}_e$. Again, before each criticality computation, the ancestral subgraph of the involved nodes and conditioning set is extracted and moralized. The computed scores are:

\begin{equation*}
    \begin{array}{lllll}
    \Psi^{\{\}}_{(w \to x)} = 1, & \Psi^{\{y\}}_{(w \to x)} = 1.\wideparen{3}, & \Psi^{\{\}}_{(w \leftarrow x)} = 1, & \Psi^{\{y\}}_{(w \leftarrow x)} = 1.\wideparen{3}, & \Psi^{\{\}}_{(w \to y)} = 0.\wideparen{6}, \\[5pt]
    \Psi^{\{x\}}_{(w \to y)} = 0.\wideparen{6}, & \Psi^{\{\}}_{(w \leftarrow y)} = 0.\wideparen{6}, & \Psi^{\{x\}}_{(w \leftarrow y)} = 0.\wideparen{6}, & \Psi^{\{\}}_{(x \to z)} = 0.\wideparen{6}, & \Psi^{\{\}}_{(x \leftarrow z)} = 0.\wideparen{6}, \\[5pt]
    \Psi^{\{\}}_{(y \to x)} = 0.\wideparen{6}, & \Psi^{\{w\}}_{(y \to x)} = 1.\wideparen{3}, & \Psi^{\{\}}_{(y \leftarrow x)} = 0.\wideparen{6}, & \Psi^{\{w\}}_{(x \to y)} = 1.\wideparen{3}.
    \end{array}
\end{equation*}

Since all values remain above the threshold $\theta = 0.5$, no additional edges are removed; the structure from $G^*_{(1)}$ is retained so $G^*_{(2)} = G^*_{(1)}$. The final DAG is obtained by converting the CPDAG $G^*_{(2)}$ back into a DAG, yielding
\begin{equation*}
    G^* = (V, E^*), \quad \text{with} \quad E^* = \{w \to x,\; w \to y,\; x \to z,\; y \to x\}.
\end{equation*}
This final structure represents a consensus BN that preserves essential dependencies while removing unnecessary complexity.\footnote{Since multiple DAGs can belong to the same equivalence class, this result is not unique. For instance, the alternative DAG $G^{*'} = (V, E^{*'})$ with edges $E^{*'} = \{x \to w, w \to y, z \to x, x \to y\}$ encodes the same conditional independencies and thus belongs to the same equivalence class as $G^*$.}

\subsection{Equivalence Class Analysis}  
We now analyse the equivalence classes of the input and fused DAGs by comparing the conditional independence (CI) relations each graph encodes. A DAG's equivalence class is determined by its skeleton (the underlying undirected graph) and v-structures (colliders)\footnote{Formally, the skeleton is the undirected graph $\widetilde{G} = (V,\widetilde{E})$ where $\widetilde{E} = \{(u \text{---} v) : (u \to v) \in E \lor (v \to u) \in E\}$, and a v-structure is any triple $(x,z,y)$ where $E$ contains $x \to z \leftarrow y$ with no edge between $x$ and $y$. The union of these features forms a \textit{pattern} that uniquely identifies the Markov equivalence class \cite{Koller_Friedman}.}, which defines its CI relations. We can assess whether the consensus graph retains meaningful dependencies while eliminating spurious ones by studying how these relationships evolve throughout the fusion process.

The input DAGs encode the following conditional independences:
\begin{align*}
    \text{CI}(E_1) &= \{w \perp z \mid x,\ w \perp z \mid y,\ x \perp z \mid y,\ w \perp y \mid x \}, \\
    \text{CI}(E_2) &= \{w \perp z \mid x,\ y \perp z \mid x,\ y \perp z \mid w,\ x \perp y \mid w \}, \\
    \text{CI}(E_3) &= \{w \perp z \mid x,\ y \perp z \mid x,\ w \perp y \}.
\end{align*}

During the intermediate transformations, structural modifications alter these relationships. The first step, aligning $E_1$ to the heuristic ordering $\sigma$, results in a loss of two conditional independencies, leaving
\begin{align*}
    \text{CI}(E_1^{\sigma}) &= \{w \perp z \mid x,\ w \perp y \mid x \}.
\end{align*}
The initial fused DAG $E^+$ introduces a stricter dependency structure, collapsing the previous independencies into a single constraint: 
\begin{align*}
    \text{CI}(E^+) &= \{w \perp z \mid \{x,y\} \}.
\end{align*}
Only $w$ and $z$ remain independent when both $x$ and $y$ are conditioned upon, with almost all conditional independences removed.

Refining the initial fusion with the MCBNC algorithm helps recover key relationships that better represent the input networks. After the first and second iterations, structures $G^*_{(1)}$ and $G^*_{(2)}$, as well as the final DAG $G^*$ have 
\begin{align*}
    \text{CI}(G^*_{(1)}) =  \text{CI}(G^*) &= \{w \perp z \mid x,\ y \perp z \mid x \},
\end{align*}
restoring the only two conditional independencies that are repeated among the input DAGs, appearing $w \perp z \mid x$ on $E_1, E_2$ and $E_3$; and $y \perp z \mid x$ on $E_2$ and $E_3$. These represent the most stable shared constraints across the input networks, reinforcing that the consensus graph should preserve only widely supported (in)dependencies. This leads to a final consensus DAG that is both compact and representative, avoiding overfitting to any single input network while maintaining interpretability and usability in real-world cases.

\end{document}